\documentclass[10pt]{article}

\usepackage[utf8]{inputenc} 
\usepackage[T1]{fontenc}    
\usepackage[numbers, compress]{natbib}
\usepackage{amsmath,amssymb,amsthm,mathtools}
\usepackage{amsfonts}

\usepackage{amsmath,amsfonts,bm}

\def\eqref#1{equation~\ref{#1}}

\def\1{\bm{1}}

\DeclareMathAlphabet{\mathsfit}{\encodingdefault}{\sfdefault}{m}{sl}
\SetMathAlphabet{\mathsfit}{bold}{\encodingdefault}{\sfdefault}{bx}{n}

\newcommand{\E}{\mathbb{E}}

\newcommand{\R}{\mathbb{R}}

\usepackage{url}
\usepackage{hyperref}
\usepackage{cleveref}
\usepackage{bbm}
\usepackage{array}
\usepackage{booktabs}
\usepackage{enumitem}
\usepackage{algorithm}
\usepackage{algpseudocode}
\usepackage{todonotes}
\usepackage{subcaption}
\usepackage{graphicx}
\usepackage{nicefrac}
\usepackage{microtype}
\usepackage{xcolor}
\usepackage{tabularx}
\usepackage{pgfplots}
\pgfplotsset{compat=1.15}
\usepackage{setspace}
\usepackage{geometry}
\usepackage{appendix}

\newcommand{\blue}[1]{\textcolor{blue}{#1}}

\newcommand{\cS}{\mathcal{S}}
\newcommand{\cA}{\mathcal{A}}

\newcommand{\cB}{\mathcal{B}}
\newcommand{\cF}{\mathcal{F}}
\newcommand{\cH}{\mathcal{H}}

\newcommand{\cP}{\mathcal{P}}
\newcommand{\cR}{\mathcal{R}}

\newcommand{\V}{\mathbb{V}}

\newcommand{\set}[1]{\left\{{#1}\right\}}
\newcommand{\bracket}[1]{\left( #1 \right)}
\newcommand{\sqbk}[1]{\left[ #1 \right]}

\newtheorem{proposition}{Proposition}
\newtheorem{theorem}{Theorem}
\newtheorem{lemma}{Lemma}
\newtheorem{corollary}{Corollary}
\newtheorem{remark}{Remark}

\title{Optimal Multi-Reward Reinforcement Learning}

\author{%
  Zijun Chen\\
  Department of Computer Science and Engineering \\
  Hong Kong University of Science and Technology \\
  \texttt{zchendg@connect.ust.hk} \\
  \and
  Zihan Zhang\thanks{Corresponding author}\\
  Department of Computer Science and Engineering \\
  Hong Kong University of Science and Technology \\
  \texttt{zsubfunc@outlook.com} \\
}

\begin{document}

\maketitle

\begin{abstract}
We study an unknown-transition finite-horizon Markov decision process (MDP) with a finite collection of known reward functions $\{r^1, r^2, \ldots, r^M\}$. The goal is to output an $\epsilon$-optimal policy for every reward using online episodic interaction only. Performance is measured by the policy error $V_{0}^{*, m} - V_{0}^{\widehat\pi^{m}, m}$ where $m\in [M]$ represents the reward function and $V_{0}^{*, m}=\mathbb{E}_{s_1\sim \mu}[V_{1}^{*, m}(s_1)]$. Under this setting, we design a provably efficient algorithm to establish a minimax sample complexity bound of
$$
O\bracket{\frac{SAH^3}{\epsilon^2}\log M \mathrm{polylog}\left(\frac{SAH\log M}{\min\set{\epsilon, 1}\delta}\right)}$$
episodes, with no additional burn-in cost. This matches the information-theoretic lower bound up to a factor of $  \mathrm{polylog}(SAH\log M/(\min\set{\epsilon, 1}\delta))$.
Our method combines three technical ingredients. First, we adapt \textsc{MVP}~\citep{zhang2024settling} to reward-switching learning to construct optimistic value estimates. Second, we use fresh replay samples to conservatively evaluate the candidate policies. Third, gap-based multiplicative weights updates adjust the reward-sampling distribution using the differences between these estimates, converting weighted learning progress into simultaneous guarantees for all rewards.
\end{abstract}

\section{Introduction}\label{sec:intro}

Reinforcement learning (RL) studies sequential decision-making problems in which an agent interacts with an environment, observes the resulting states, and chooses actions to maximize a cumulative reward. In the standard finite-horizon formulation, the environment is modeled as a Markov decision process (MDP) with finite state-action space $\mathcal{S}\times \mathcal{A}$ and horizon $H$.
Under this setting, a substantial literature has developed optimistic and variance-sensitive algorithms with near-optimal regret and sample-complexity guarantees \citep{azar2017minimax,jin2018qlearningprovablyefficient,zhang2024settling}.

However, these results typically focus on a single reward function, an assumption that is restrictive in many applications. A fixed environment may be evaluated according to several objectives, such as safety, efficiency, fairness, or user preference, while the underlying transition dynamics remain unchanged. For example, a recommender system may be required to optimize different notions of user utility, and a control system may need to support multiple task specifications without relearning the environment from scratch. The same phenomenon appears in reward-mixing MDPs and contextual decision processes, where a common environment is used to study multiple objectives or task-dependent criteria \citep{hallak2015contextual}. Since all reward functions share the same transition kernel, trajectories collected for one objective may contain useful information for the others. Exploiting this shared structure is therefore both statistically natural and practically important.

At the other extreme, reward-free or reward-agnostic exploration aims to collect data that supports policy learning for rewards that are unknown during exploration or specified only afterward \citep{jin2020reward,li2026minimax}. Such formulations require broad coverage of the MDP and can be more demanding than the setting considered here. We study an intermediate but fundamental problem: the learner is given $M$ different reward functions in advance and may use them during exploration, but it must eventually return an approximately optimal policy for every reward. A naive approach that runs a separate RL algorithm for each reward incurs a factor proportional to $M$, whereas an overly reward-agnostic strategy may collect substantially more data than is necessary for the given reward class.

Recent work by \citet{zhang2024settling} established that a near-optimal policy for a single reward function can be learned with the minimax-optimal sample complexity
$\widetilde{O}(SAH^3/\epsilon^2)$, without  additional burn-in terms. This naturally raises the following fundamental question:
\begin{center}
\emph{Can we simultaneously learn near-optimal policies for $M$ different reward functions using
$\widetilde{O}\bigl(
\frac{SAH^3\log M}{\epsilon^2}
\bigr)$ episodes, while still avoiding any additional burn-in terms?}
\end{center}

\subsection{Our Contribution}
We answer this question affirmatively.  More precisely, we design an online algorithm such that: for any fixed failure probability $\delta\in (0,1)$ and error threshold $\epsilon\in (0,H]$, with probability at least $1-\delta$, the algorithm returns an $\epsilon$-optimal policy for every reward function using
\[
O\!\left(
\frac{SAH^3}{\epsilon^2}
\log M\,
\mathrm{polylog}\!\left(
\frac{SAH\log M}{\min\set{\epsilon, 1}\delta}
\right)
\right)
\]
episodes (see the formal statement in Theorem~\ref{thm:adv_sample_complexity}).
This sample-complexity bound shows that policies for multiple reward functions can be learned jointly and efficiently, with at most an additional $\log M$ factor over the minimax-optimal sample complexity for a single reward function.
On the other hand, we adapt the recent lower-bound construction for
reward-agnostic reinforcement learning \citet{ridel2026improvedboundsrewardagnosticrewardfree} to show that
any algorithm for simultaneously learning near-optimal policies for $M$ reward functions requires
$\Omega\!\left(SAH^3\log(M)/\epsilon^2\right)$ exploration episodes; see
Appendix~\ref{app:lb} for the detailed construction. Consequently, when $\log M \leq S$, our upper bound matches
the lower bound in Corollary~\ref{cor:finite-known-reward-lb-M}
up to a multiplicative factor of
$\mathrm{polylog}(SAH\log M/(\min\set{\epsilon,1}\delta))$.

\subsection{Related Works}

\paragraph{Tabular MDPs.}
A large body of theoretical work on online reinforcement learning considers tabular MDPs in which all episodes share the same transition kernel, reward function, and action space. Over the past several decades, substantial progress has been made toward improving the sample efficiency of both model-based and model-free algorithms. Representative results include regret and PAC guarantees based on optimism, posterior sampling, and variance-aware exploration
\citep{zhang2022near,agrawal2017posterior,
szita2010model,dann2017unifying,
domingues2021episodic,azar2017minimax,
efroni2019tight,neu2020unifyingviewoptimismepisodic},
as well as sample-efficient model-free methods
\citep{jin2018qlearningprovablyefficient,dong2019q,
menard2021ucb,
li2021breaking,
li2022settling,
wang2020longhorizonreinforcementlearning}.
Another line of work develops problem-dependent and gap-dependent guarantees
\citep{zanette2019tighter,
simchowitz2019nonasymptoticgapdependentregretbounds,
pmlr-v130-yang21b,xu2021finegrainedgapdependentboundstabular,
NEURIPS2025_a2fb8d64}.
Most closely related to our analysis, \citet{zhang2024settling} established the all-range minimax regret bound $
\widetilde{O}(
\min\{\sqrt{SAH^3K},\,HK\}
)$, 
which implies the minimax-optimal sample complexity
$\widetilde{O}(SAH^3/\epsilon^2)$ without an additional burn-in term.

\paragraph{Reward-agnostic reinforcement learning.} 
In reward-agnostic RL, the learner explores the environment without
access to the downstream reward functions, which belong to a fixed class of
cardinality $M$ chosen independently of the exploration data. This framework
therefore naturally extends to the simultaneous learning of policies for
multiple rewards. In this setting,
\citet{zhang2020taskagnostic} first established a sample-complexity bound of
\(\widetilde{O}\!\left(SAH^5\log(M)/\epsilon^2\right)\).
Subsequent work by
\citet{li2026minimax}
improved the dependence on the horizon, obtaining
\(\widetilde{O}\!\left(SAH^3\mathrm{polylog}(M)/\epsilon^2
+ S^4A^4H^6\mathrm{polylog}(M)/\epsilon\right)\)
for reward classes of polynomial cardinality. More recently,
\citet{ridel2026improvedboundsrewardagnosticrewardfree} further improved the burn-in dependence, yielding
a sample-complexity bound of
\(\widetilde{O}\!\left(SAH^3\mathrm{polylog}(M)/\epsilon^2
+ S^2AH^4\mathrm{polylog}(M)/\epsilon\right)\).\footnote{The guarantees of
\cite{li2026minimax,
ridel2026improvedboundsrewardagnosticrewardfree} cover fixed reward classes of cardinality polynomial in $S,A,H$.
We write $\mathrm{polylog}(M)$ conservatively: replacing $\delta$ by $\delta/M$
inherits the original powers of the confidence logarithms, which need not equal one.}
Thus, for fixed reward classes of polynomial cardinality, the improved
reward-agnostic bound has a leading term of
$\widetilde{O}\!\left(SAH^3\mathrm{polylog}(M)/\epsilon^2\right)$
and an additional burn-in term of order
$\widetilde{O}\!\left(S^2AH^4\mathrm{polylog}(M)/\epsilon\right)$. We refer the reader to Table~\ref{table1} for a detailed comparison.

\begin{table*}[t]
    \centering
    \caption{Comparison of online episodic sample-complexity
    bounds for learning near-optimal policies for $M$ reward functions.}
    \label{table1}
    \renewcommand{\arraystretch}{2}
    \begin{tabular}{lll}
        \hline
        Reference & Sample complexity & Comment \\
        \hline

        \cite{jin2020reward}
        &
        $\widetilde{O}\!\left(
            \frac{S^2AH^5}{\epsilon^2}
            + \frac{S^4AH^7}{\epsilon}
        \right)$
        &
        Reward-free exploration.
        \\

        \cite{zhang2020taskagnostic}
        &
        $\widetilde{O}\!\left(
            \frac{SAH^5\log M}{\epsilon^2}
        \right)$
        &
        --
        \\
        \cite{menard2021fast}
        &
        $\widetilde{O}\!\left(
            \frac{S^2AH^3}{\epsilon^2}

        \right)$
        &
        Reward-free exploration.
        \\
        \cite{li2026minimax}
        &
        $\widetilde{O}\!\left(
            \frac{SAH^3\log M}{\epsilon^2}
            + \frac{S^4A^4H^6\log M}{\epsilon}
        \right)$
        &
        --
        \\

        \cite{ridel2026improvedboundsrewardagnosticrewardfree}
        &
        $\widetilde{O}\!\left(
            \frac{SAH^3\log M}{\epsilon^2}
            + \frac{S^2AH^4\log M}{\epsilon}
        \right)$
        &
        --
        \\

        \textbf{This work}
        &
        \blue{$\widetilde{O}\!\left(
            \frac{SAH^3\log M}{\epsilon^2}
        \right)$}
        &
        No burn-in terms.
        \\

        \hline

        Lower bound
        &
        {$\tilde{\Omega}\!\left(
            \frac{SAH^3\min\set{\log M, S}}{\epsilon^2}
        \right)$}
        &
        See Appendix~\ref{app:lb}.
        \\

        \hline
    \end{tabular}

    \medskip
    The lower-bound construction adapts those of
    \citet{jin2020reward,ridel2026improvedboundsrewardagnosticrewardfree}.
\end{table*}

\section{Preliminaries}
\label{sec:preliminary}
In this section, we introduce the basics of finite-horizon MDPs, as well as the basic episode setting.

\paragraph{Basics of finite-horizon MDPs.} This paper focuses on time-inhomogeneous finite-horizon MDPs. Let $\cS$ be a finite state space with $|\cS| = S$, and let $\cA$ be a finite action space with $|\cA|= A$, and let $H$ be the horizon. The notation $P=\set{P_h(\cdot |s,a)}_{(h,s,a)\in [H]\times\cS\times \cA}$ denotes the transition kernel. If action $a$ is taken, then the state at the next step $h+1$ of the environment is randomly drawn from $P_{h,s,a}:=P_h(\cdot |s,a)\in \Delta(\cS)$ and receives the reward $r_h(s,a)\in [0,1]$
instantaneously. A policy $\pi = \{\pi_h\}_{h=1}^H$ consists of functions that map each state to a probability distribution over the action space. When $\pi$ is deterministic, we write $\pi_h(s)$ for the action taken at state $s$ in layer $h$. We also allow episode-level mixtures of such policies.

\paragraph{Multi-Reward Setting.} We consider the multi-reward setting with $M$ known reward functions. Specifically, the reward function is denoted by $r^{m}=\{r_{h}^{m}(s,a)\}_{(h,s,a)\in [H]\times \cS\times \cA}$ where $r_{h}^{m}(s,a)\in [0,1]$ for all $(h,s,a,m)\in [H]\times \cS\times \cA\times [M]$. All rewards share the same transition kernel of the underlying finite-horizon MDP. We state our results for $M\ge2$; the single-reward case is covered by \citet{zhang2024settling}.

\paragraph{Value function and Q-function.} Define the reward-specific value function $V^{\pi, m}$ and the Q-function $Q^{\pi, m}$ as
\begin{align*}
    V^{\pi, m}_{h}(s) :=& \E^{\pi, m}\sqbk{\sum_{h'=h}^{H} r_{h'}^{m}(s_{h'}, a_{h'}) | s_h = s}, &\forall& (h,s)\in [H]\times \cS\\
    Q^{\pi, m}_{h}(s,a) :=& \E^{\pi, m}\sqbk{\sum_{h'=h}^{H} r_{h'}^{m}(s_{h'}, a_{h'}) | s_h = s, a_h = a}, &\forall& (h,s,a)\in [H]\times \cS\times\cA
\end{align*}
where the expectation $\E^{\pi ,m}\sqbk{\cdot}$ is taken over the randomness of an episode $\set{(s_h, a_h)}_{h=1}^H$ generated under policy $\pi$ and reward function $r^{m}$. Denote $V_0^{\pi, m} =\mathbb{E}_{s_1\sim \mu}[V_1^{\pi,m}(s_1)] $ where $\mu$ is the initial distribution. Accordingly, we define the optimal reward-specific value function and Q-function as
\begin{align*}
    V_h^{*, m}(s):=& \max_{\pi}V_h^{\pi, m}(s), & \forall& (h,s)\in [H]\times\cS \\
    Q_h^{*, m}(s,a):=& \max_{\pi} Q_{h}^{\pi, m}(s,a), &\forall&(h,s,a)\in [H]\times \cS\times\cA
\end{align*}
and $V_0^{*,m} = \max_{\pi}V_0^{\pi,m}$.

\paragraph{The learning problem.} Let $V_{0}^{*, m}$ be the optimal expected value function subject to the reward function $r^{m}$ and initial distribution $\mu$, and $V_{0}^{\pi, m}$ be the value of $\pi$ in $r^{m}$. In the multi-reward setting, each reward function $r^m$ induces its own optimal policy. The PAC sample complexity quantifies the number of episodes required to learn an $\epsilon$-optimal policy collection $\{\widehat{\pi}^{m}\}_{m\in [M]}$ satisfying
$$
V_{0}^{*, m} - V_{0}^{\widehat{\pi}^{m}, m} \leq \epsilon, \quad \forall m\in [M]
$$
simultaneously with high probability.

\paragraph{Notations.} We denote $\V(P, X):=\langle P, X^2 \rangle -\langle P, X\rangle^2$ as the variance on random variable $X$ on distribution $P$, where $X^2$ is element-wise square of $X$. We use $[N]$ to denote the set $\{1,2,\ldots, N\}$.




\section{Algorithm}\label{sec:alg}
We introduce $\texttt{Online-Multi-Reward-Learning-with-Exponential-Weight}$ (Algorithm \ref{alg:main}), which adapts the \texttt{MVP} framework in~\citet{zhang2020reinforcement} to the multi-reward MDP setting by incorporating three novel techniques: \emph{shared-model multi-reward learning}, \emph{sealed replay}, and \emph{multiplicative weight update}. In each phase $\ell$, we generate candidate policies $\{(\pi^{\ell, k, m})_{k\in [K]}\}_{m\in [M]}$ and optimistic estimate $\hat{v}_{\ell, m}^{+}$ using $\texttt{Multi-Reward-Switching-MVP}$ (Algorithm~\ref{alg:mvp}). Then, $\texttt{Sealed-Replay}$ (Algorithm~\ref{alg:replay}) collects fresh samples to construct the retained empirical kernel $\widetilde{P}$. Finally, we evaluate the candidate policies under $\widetilde{P}$ using $\texttt{Conservative-Evaluation}$ (Algorithm~\ref{alg:pessimistic_evaluation}) to derive the conservative estimate $\hat{v}_{\ell, m}^{-}$, and use the estimate gap $\hat{v}_{\ell, m}^{+} - \hat{v}_{\ell, m}^{-}$ to update the reward sample distribution $p_{\ell}$ with multiplicative weight.

\paragraph{Shared-model multi-reward learning.} At the beginning of phase $\ell$, we fix the reward-sampling distribution $p_{\ell}$ and run $\texttt{Multi-Reward-Switching-MVP}$, the multi-reward MVP subroutine (Algorithm~\ref{alg:mvp}), for $K$ episodes. At the beginning of each episode $k$, following prior work~\citet{zhang2024settling}, the subroutine first computes, for every reward index $m\in[M]$, an optimistic value function $\overline{V}_{1}^{\ell,k,m}$ and the corresponding greedy counterfactual policy $\pi^{\ell,k,m}$, using the shared empirical transition kernel $\widehat{P}$ and the reward-specific exploration bonuses $b_{h}^{\ell,k,m}(s,a)$. The agent then samples a reward index $m^{\ell,k}\sim p_{\ell}$ and executes the roll-out policy $\pi^{\ell,k}:=\pi^{\ell,k,m^{\ell,k}}$.

The empirical transition model and the statistics maintained under the doubling schedule are carried over across phases, allowing learning for all rewards to benefit from previously collected transitions. The subroutine records all counterfactual policies $\{\pi^{\ell,k,m}\}_{k\in[K],\,m\in[M]}$, the roll-out policies $(\pi^{\ell,k})_{k\in[K]}$, and the initial states $(s_{1}^{\ell,k})_{k\in[K]}$. It returns the average optimistic estimates $\hat{v}_{\ell,m}^{+}$ for all $m\in[M]$ (line~\ref{line:alg:optimism}), together with the visit counts $N_{\ell,h}(s,a)$ accumulated only during the current phase.

\paragraph{Sealed replay.} In line~\ref{line:alg:main:sealed_replay}, we invoke $\texttt{Sealed-Replay}$ (Algorithm~\ref{alg:replay} in Appendix~\ref{sec:app:algcode}) to execute each of the $K$ roll-out policies $(\pi^{\ell,k})_{k\in[K]}$ collected during the learning stage for a constant $8$ times, yielding $8K$ replay episodes in total. Only tuples $(h,s,a)$ whose visit counts $N_{\ell,h}(s,a)$, accumulated during that learning stage (Algorithm~\ref{alg:mvp}), meet or exceed the threshold specified in line~\ref{line:alg:main:retained_set} are retained, denoted as $\mathcal{R}_{\ell} =\left\{ (h,s,a): N_{\ell,h}(s,a) \geq 8 \log \frac{1}{\delta''}\right\}$. For each retained tuple $(h,s,a)$, $\texttt{Sealed-Replay}$ stores the first $\widetilde N_{\ell,h}(s,a)$ replay samples, capped at $N_{\ell,h}(s,a)$. It then constructs a fresh empirical transition model $\widetilde{P}_{\ell}$ using only these replay samples. Replay samples are not used to update the learning model and are discarded at the end of phase $\ell$.

\paragraph{Conservative evaluation.} Using the sealed replay samples collected in phase $\ell$, we invoke $\texttt{Conservative-Evaluation}$ (Algorithm~\ref{alg:pessimistic_evaluation} in Appendix~\ref{sec:app:algcode}) to compute a conservative value estimate for each counterfactual policy $\pi^{\ell,k,m}$, $k\in[K]$, $m\in[M]$, under the corresponding reward function $r^m$.
For each retained row $(h,s,a)\in\cR_\ell$, starting from $\underline{V}_{H+1}^{\ell,k,m}\equiv 0$, we compute the evaluation
backup
\begin{equation}
\underline{Q}_{h}^{\ell,k,m}(s,a)
=
r_h^m(s,a)
+
\frac{\widetilde{N}_{\ell,h}(s,a)}{\widetilde{N}_{\ell,h}(s,a)+1}
\left\langle
\widetilde{P}_{\ell,h,s,a},
\underline{V}_{h+1}^{\ell,k,m}
\right\rangle.
\end{equation}
For all other triples not in $\mathcal{R}_{\ell}$, we set
$\underline{Q}_{h}^{\ell,k,m}(s,a)=r_h^m(s,a)$.
The value recursion follows the corresponding counterfactual
policy, with
$\underline{V}_{h}^{\ell,k,m}(s)
=\sum_{a\in\cA}\pi_h^{\ell,k,m}(a\mid s)
\underline{Q}_{h}^{\ell,k,m}(s,a)$.
Averaging the resulting values at the recorded initial states
yields the conservative estimate
$\hat{v}_{\ell,m}^{-}
=\frac{1}{K}\sum_{k=1}^{K}
\underline{V}_{1}^{\ell,k,m}(s_1^{\ell,k})$.

\begin{remark}
A key feature of this conservative evaluation is that it does not
subtract an explicit pessimistic bonus. Instead, it shrinks the
continuation value in each backward backup according to the visit
count of the corresponding triple $(h,s,a)$.
This design offers two advantages: (i) it keeps the mean
underestimation bias small; and
(ii) it controls upward fluctuations even when $N_{\ell,h}(s,a)$
is small, so that $\hat{v}_{\ell,m}^{-}$ cannot exceed the true value very often.
\end{remark}

\paragraph{Gap-based multiplicative weights update.}
Finally, we update the reward-sampling distribution using the gap between the optimistic and conservative estimates:
\begin{align}
    p_{\ell+1}(m) \propto p_\ell(m)\exp\!\left(\frac{\hat{v}_{\ell,m}^{+}-\hat{v}_{\ell,m}^{-}}{H}\right).\label{eq:updatep}
\end{align}
Thus, rewards with larger certificate gaps receive larger multiplicative weight increases relative to rewards with smaller gaps.

\paragraph{Output.} After $L$ phases, we return, for each $m\in[M]$, the uniform mixture of the counterfactual policies indexed by $[L]\times[K]$ (line~\ref{line:alg:mixture_policy}) $\overline{\pi}^{m} =\mathrm{Unif}((\pi^{\ell,k,m})_{\ell\in[L],\,k\in[K]})$. Each phase uses $K$ learning episodes and $8K$ replay episodes, for a total of $9KL$ episodes of environment interaction.

{\small
\begin{algorithm}[ht]
\caption{$\texttt{Online-Multi-Reward-Learning-with-Exponential-Weight}\\$$(\cS, \cA, H, K,L,\{r^{m}\}_{m\in [M]},\epsilon,\delta)$\label{alg:main}}
\begin{algorithmic}[1]
	\State {\textbf{input:} state space $\mathcal{S}$, action space $\mathcal{A}$, horizon $H$, learning episodes per phase $K$, phases $L$, reward functions $\{r^{m}\}_{m\in [M]}$, accuracy $\epsilon\in(0,H]$, confidence parameter $\delta\in(0,1)$}
    \State{Set $\delta'=\frac{\delta}{64SAHKLM}$, $\delta''=\frac{\delta}{64SAHL}$, and $p_1(m)\gets1/M$ for all $m\in[M]$.}
    \State{For all $(h,s,a,s')$, set $N^{\mathsf{all}}_h(s,a),N_h(s,a,s'),N_h(s,a)\gets0$; initialize $\widehat P_{h,s,a}\in \Delta(\cS)$.}
	\For{phase $\ell = 1, 2, \ldots, L$}
		\State {$N_\ell, \{\hat{v}_{\ell,m}^{+}\}_{m\in[M]}, (\pi^{\ell,k})_{k\in[K]}, \{(\pi^{\ell,k,m})_{k\in [K]}\}_{m\in[M]}, (s_{1}^{\ell,k})_{k\in[K]}\gets \texttt{Multi-Reward-Switching-MVP}(p_\ell)$; }\label{line:alg:main:rs_mvp}
		\State {$\cR_\ell \gets \{(h,s,a): N_{\ell,h}(s,a)\geq 8\log \frac{1}{\delta''}\}$ {\color{blue}\Comment{get the retained rows}}}\label{line:alg:main:retained_set}
		\State {$\widetilde{P}_\ell, \widetilde{N}_\ell\gets \texttt{Sealed-Replay}(\cR_\ell, N_\ell, (\pi^{\ell,k})_{k\in[K]})$;{\color{blue}\Comment{fresh replay model}}}\label{line:alg:main:sealed_replay}
        \For{$m\in[M]$}\label{line:alg:pessimistic_evaluation_starts}
		\For{$k = 1,2,\ldots,K$}
			\State{$\underline{V}_{1}^{\ell,k,m}\gets \texttt{Conservative-Evaluation}(\widetilde{P}_\ell, \widetilde{N}_\ell, \cR_\ell, \pi^{\ell,k,m}, r^{m})$;}\label{line:alg:main:pessimistic_evaluation}
		\EndFor
        \State{$\hat{v}_{\ell,m}^{-}\gets \frac{1}{K}\sum_{k=1}^{K} \underline{V}_{1}^{\ell,k,m}(s_{1}^{\ell,k})$}\label{line:alg:pessimism}
        \EndFor \label{line:alg:pessimistic_evaluation_ends}
        \State{For all $m\in[M]$, update {\color{blue}\Comment{multiplicative weight update}}
        \[
        p_{\ell+1}(m)\gets
        \frac{p_\ell(m)\exp\bigl((\hat v_{\ell,m}^{+}-\hat v_{\ell,m}^{-})/H\bigr)}
        {\sum_{j=1}^{M}p_\ell(j)\exp\bigl((\hat v_{\ell,j}^{+}-\hat v_{\ell,j}^{-})/H\bigr)}.
        \]}\label{line:alg:main:certificate}
	\EndFor
    \State{Set $\overline\pi^m\gets\mathrm{Unif}((\pi^{\ell,k,m})_{\ell\in[L],\,k\in[K]})$ for all $m\in[M]$.}\label{line:alg:mixture_policy}
    \State{\textbf{output:} $\{\overline\pi^m\}_{m\in[M]}$}
\end{algorithmic}
\end{algorithm}
}

{\small
\begin{algorithm}[ht]
\caption{$\texttt{Multi-Reward-Switching-MVP}(p)$}
\label{alg:mvp}
\small
\begin{algorithmic}[1]
    \State{\textbf{input:} reward-sampling distribution $p$ over $[M]$, $c_1=\frac{460}{9}$, $c_2=\frac{544}{9}$.
	}
	\State{\textbf{shared state:} $N^{\mathsf{all}}$, $N_h(s,a,s')$, $N_h(s,a)$, and $\widehat P$, initialized once in Algorithm~\ref{alg:main} and updated in place.}
    \State{Set all entries of $N_\ell=\{N_{\ell,h}(s,a)\}_{(h,s,a)\in[H]\times\cS\times\cA}$ to zero.}
	\For{$k=1,2,\cdots, K$}
		\State{Set $\overline{V}_{H+1}^{k, m}(s)\gets 0$ for all $s\in \cS, m \in [M]$;}
		\For{$h=H, H-1, \cdots, 1$}
			\For{$(s,a)\in \cS\times \cA$ and $m \in [M]$}
			\vspace{-0.5em}
				\small{\begin{align}
					\label{eq:update1}
				b_{h}^{k, m}(s,a) &\leftarrow c_1 \sqrt{\frac{\mathbb{V}(\widehat{P}_{h,s,a} ,\overline{V}_{h+1}^{k, m}) \log \frac{1}{\delta'} }{ \max\{N_h(s,a),1 \}}} +c_2\frac{H\log \frac{1}{\delta'}}{\max\{N_h(s,a) ,1\}  },\\
				\overline{Q}_h^{k, m}(s,a) &\leftarrow \min\big\{r_h^{m}(s,a)+\langle \widehat{P}_{h,s,a}, \overline{V}_{h+1}^{k, m} \rangle +b_{h}^{k, m}(s,a)    ,H\big\}.
			\end{align}}
			\vspace{-4ex}
            \EndFor
				\State{Set $\overline{V}_{h}^{k,m}(s)\gets\max\limits_{a\in\cA}\overline{Q}_{h}^{k,m}(s,a)$ and $\pi_h^{k,m}(s)\in\arg\max\limits_{a\in\cA}\overline{Q}_{h}^{k,m}(s,a)$ for all $(s,m)\in\cS\times [M]$.}
		\EndFor
		\State{Draw $m^k\sim p$, set $\pi^k \gets \pi^{k, m^k}$; start a fresh episode with $s_1^k\sim\mu$. {\color{blue}\Comment{roll-out policy}}}
		\For{$h=1,2,\cdots, H$}
			\State{Observe $s_{h}^{k}$, execute $a_{h}^{k} = \pi_{h}^{k}(s_{h}^{k})$, receive $r_{h}^{k}$, observe $s_{h+1}^{k}$. $(s,a,s')\gets (s_{h}^{k}, a_{h}^{k}, s_{h+1}^{k})$;}
			\State	{Update $N^{\mathsf{all}}_h(s,a) \leftarrow  N^{\mathsf{all}}_h( s,a )+1$, $N_h(s,a,s') \leftarrow N_h(s,a,s')+1$, $N_{\ell,h}(s,a)\leftarrow N_{\ell,h}(s,a)+1$.}
			\If{$N^{\mathsf{all}}_h(s,a)\in \{2^j:j=0,1,\ldots,\lfloor\log_2(KL)\rfloor\}$ \label{line:rp_update_start} }   \label{line:trigger-set}
			\State{ $N_h(s,a)\leftarrow \sum_{\widetilde{s}}N_h(s,a,\widetilde{s})$;  
			{\color{blue}\Comment{number of visits to $(h,s,a)$ in this epoch.} \label{line:Nh-update}}}
		\State{Set $\widehat{P}_{h,s,a,\tilde{s}} \leftarrow  \frac{N_h(s,a,\widetilde{s})}{N_h(s,a)}$ and $N_h(s,a,\widetilde{s})\leftarrow 0$ for all $\widetilde{s} \in \mathcal{S}$.  
  \label{line:P-hsa-update}}
		\EndIf
		\EndFor
	\EndFor
    \State {Set $\hat v_{m}^{+} \gets \frac{1}{K} \sum_{k=1}^{K} \overline{V}_{1}^{k, m}(s_1^{k})$ for all $m\in[M]$; {\color{blue}\Comment{optimistic estimate}}}\label{line:alg:optimism}
	\State {\textbf{return:} $N_\ell$, $\{\hat{v}_m^+\}_{m\in [M]}$, $(\pi^k)_{k\in [K]}$, $\{(\pi^{k, m})_{k\in [K]}\}_{m\in [M]}$, $(s_{1}^{k})_{k\in [K]}$}
\end{algorithmic}
\end{algorithm}
}

\begin{samepage}
\section{Sample Complexity Analysis}\label{sec:minimax}
In this section, we present the statement and proof sketch of the main theorem.

\begin{theorem}[Sample Complexity of Algorithm~\ref{alg:main}]\label{thm:adv_sample_complexity}
There exists a universal constant $C>0$ such that, for any $M\ge2$, $SAH\ge2$, $\delta\in(0,1)$, and $\epsilon\in(0,H]$, setting $\iota=\log^4\frac{SAH\log M}{\min\set{\epsilon,1}\delta}$, $K=\lceil C\tfrac{SAH^2\iota^{3/4}}{\epsilon}\rceil$ learning episodes per phase, and $L=\lceil C\tfrac{H\iota^{1/4}}{\epsilon}\log M\rceil$ phases in Algorithm~\ref{alg:main} ensures that, with probability at least $1-\delta$, the algorithm returns a policy collection $\{\overline{\pi}^{m}\}_{m\in[M]}$ such that $\overline{\pi}^{m}$ is $\epsilon$-optimal for every $m\in[M]$. The total number of episodes, including learning and replay, is at most
\[
9KL = O\bracket{\frac{SAH^3\iota\log M }{\epsilon^2}} = O\bracket{\frac{SAH^3\log M}{\epsilon^2} \log^4\bracket{\frac{SAH\log M}{\min\set{\epsilon, 1}\delta}}}.
\]
\end{theorem}
\end{samepage}
Theorem~\ref{thm:adv_sample_complexity} extends the burn-in-free sample-complexity guarantee of \textsc{MVP}~\citep{zhang2024settling} to the multi-reward setting, with optimal logarithmic dependence on $M$ up to additional logarithmic factors.\footnote{This dependence is optimal in the regime $M\lesssim \exp(S)$. For $M\gtrsim \exp(S)$, the reward-free algorithm of~\citet{menard2021fast} achieves an optimal  sample complexity of $\widetilde{O}(S^2AH^3/\epsilon^2)$ episodes.}

\paragraph{Additional notations.}
 For phase $\ell$, we write $\overline{V}^{\ell,k,m}$ and $\underline{V}^{\ell, k, m}$ for the optimistic value function $\overline{V}^{k,m}$ and conservative value function $\underline{V}^{\ell, k, m}$ when evaluating $\pi^{\ell, k, m}$ in phase $\ell$ under $r^{m}$ respectively. We use the same convention for other phase-dependent quantities, including $\overline{Q}^{\ell, k, m}$, $\underline{Q}^{\ell, k, m}$, $m^{\ell, k}$, $s^{\ell, k}_{1}$ and let $
v_{\ell,m}:=\frac1K\sum_{k=1}^K V_1^{\pi^{\ell,k,m},m}(s_1^{\ell,k}).$ 
Recall that $\hat{v}^+_{\ell,m}$ is the optimistic estimate computed in
line~\ref{line:alg:main:rs_mvp} of Algorithm~\ref{alg:main}, whereas
$\hat{v}^-_{\ell,m}$ is the conservative estimate computed in
line~\ref{line:alg:pessimism} of Algorithm~\ref{alg:main}.

\paragraph{High-level idea.}
To establish $\epsilon$-optimality for all $M$ rewards, it suffices,
up to an initial-state concentration term, to control the worst-reward
average optimistic surplus $
\max_{m\in[M]}\frac{1}{L}\sum_{\ell=1}^L
(\hat{v}^+_{\ell,m}-v_{\ell,m}).$

For this high-level discussion, we consider the high-probability
optimism event, on which $\hat{v}^+_{\ell,m}\ge v_{\ell,m}$ for all
$\ell,m$.
Our strategy is to estimate these surpluses using the gaps
$\hat{v}^+_{\ell,m}-\hat{v}^-_{\ell,m}$ and apply multiplicative weights
updates to convert gap-weighted progress into simultaneous
guarantees for every reward. Two difficulties arise along this approach.

First, accurate phase-wise pessimism is expensive.
If $\hat{v}^-_{\ell,m}\le v_{\ell,m}$ held in every phase, each
estimated gap would directly upper-bound the corresponding optimistic
surplus. A conventional approach would enforce this inequality through
pessimistic Bellman backups with confidence bonuses. However, keeping
the resulting confidence widths sufficiently small in every phase can
require substantially more replay samples than our intended per-phase
budget of $K=\widetilde{O}(SAH^2/\epsilon)$ episodes.

Second, a generic multiplicative-weights regret bound is too loose.
For arbitrary signed gaps in $[-H,H]$, the standard worst-case analysis,
with an appropriately tuned learning rate, gives
\begin{align*}
&\max_{m\in[M]}\frac{1}{L}\sum_{\ell=1}^L
\left(\hat{v}^+_{\ell,m}-\hat{v}^-_{\ell,m}\right)
-
\frac{1}{L}\sum_{\ell=1}^L\sum_{m=1}^M p_\ell(m)
\left(\hat{v}^+_{\ell,m}-\hat{v}^-_{\ell,m}\right)\lesssim H\sqrt{\frac{\log M}{L}}.
\end{align*}
This bound does not yield an $O(\epsilon)$ overhead
with $L=\widetilde{O}(H\log M/\epsilon)$ phases. 
Instead, we seek a constant-factor comparison with an additive
$O(H\log(M/\delta)/L)$ term:
\begin{align}
\max_{m\in [M]} \sum_{\ell=1}^L\frac{1}{L}\left(\hat{v}^+_{\ell,m} - \hat{v}^-_{\ell,m} \right)  \lesssim \frac{4}{L}\sum_{\ell=1}^{L}\sum_{m = 1}^M p_{\ell}(m) \left( \hat{v}^+_{\ell,m} - \hat{v}_{\ell,m}^{-} \right) + \frac{H}{L}\log\frac{64M}{\delta}.\label{eq:wa}
\end{align}
Such a comparison follows directly in the case $(\hat{v}_{\ell,m}^+ - \hat{v}_{\ell,m}^{-})\geq 0$ for all $(\ell,m)$, but our
evaluation gaps can be negative because phase-wise pessimism is not
guaranteed.

We address both difficulties by designing $\hat{v}^-_{\ell,m}$ as a
\emph{conservative evaluation} based on fresh replay samples, rather
than as a phase-wise lower confidence bound.
The evaluator is conservative in conditional mean
with a small bias.
It also satisfies a conditional exponential-moment guarantee for $\hat{v}^-_{\ell,m}-v_{\ell,m}$, which controls cumulative
overestimation.  These modifications allow us to establish an approximate version of~\eqref{eq:wa}.

The key is therefore not to eliminate individual negative gaps, but to control their effect through conditional moments and cumulative evaluation error, adding only an $O(H\log(M/\delta)/L)$ overhead.

Given~\cref{eq:wa}, it remains to control the weighted gap $\frac{1}{L}\sum_{\ell=1}^{L}\sum_{m\in[M]}p_\ell(m) \left(
\hat{v}^+_{\ell,m}-\hat{v}^-_{\ell,m}
\right).$  
We bound this quantity by combining the regret guarantee of
reward-switching \textsc{MVP}. Together with the cumulative
overestimation guarantee, we obtain that, with high probability, $
\max_{m\in[M]}\frac{1}{L}\sum_{\ell=1}^{L}
\left(\hat{v}^+_{\ell,m}-v_{\ell,m}\right)
=O(\epsilon).$

The full proof of Theorem~\ref{thm:adv_sample_complexity} is deferred to Appendix~\ref{app:minimax_additional_proof}. We sketch the key arguments below.

\subsection{Multi-Reward Learning Guarantee.}
 
Since the learning state persists across phases, optimism and the \textsc{MVP} analysis~\citep{zhang2024settling} apply over all $KL$ learning episodes.
\begin{lemma}[Optimism]
\label{lemma:optimism}
With probability at least $1-4SAHKLM\delta'$, $V_1^{*,m}(s)\le\overline V_1^{\ell,k,m}(s)$ for all $(s,\ell,k,m)\in\cS\times[L]\times[K]\times[M]$.
\end{lemma}
The difference between the optimistic estimate and the policy value is then nonnegative and bounds suboptimality. Along the executed policies, the global learning analysis gives
\[
\sum_{\ell=1}^L\sum_{k=1}^K
\big(\overline V_1^{\ell,k,m^{\ell,k}}\!(s_1^{\ell,k})
-V_1^{\pi^{\ell,k,m^{\ell,k}},m^{\ell,k}}\!(s_1^{\ell,k})\big)\lesssim\min\!\bigg\{\!\sqrt{SAH^3KL\log_2^3(2KL)\log\frac M{\delta'}},HKL\bigg\}
\]
with probability at least $1-5\delta'$. Applying conditional-mean concentration to the surpluses and then intersecting with optimism transfers the selected surplus to its weighted average.
\begin{lemma}
\label{lemma:optimistic_average_counterfactual}
With probability at least $1-(4SAHKLM+5)\delta'-\delta/64$,
\begin{equation}
\label{equ:weighted_average_optimistic_error}
\frac1L\sum_{\ell=1}^L\sum_{m=1}^M p_\ell(m)
\bracket{\hat v_{\ell,m}^+-v_{\ell,m}}
\lesssim\sqrt{\frac{SAH^3}{KL}\log_2^3(2KL)\log\frac M{\delta'}}
+\frac H{KL}\log\frac{64}{\delta}.
\end{equation}
\end{lemma}

We postpone the proofs to Appendices~\ref{app:subsubsec:proof_of_optimism} and~\ref{app:subsubsec:proof_of_optimistic_average_counterfactual}.

\subsection{Conservative Evaluation with Sealed Replay}


In each phase, we replay each of the $K$ recorded roll-out policies eight times. With high probability, this yields enough fresh samples to match the learning-phase visit counts for every retained tuple $(h,s,a)$.
\begin{lemma}
\label{lemma:replay_count_bound}
With probability at least $1-2SAHL\delta''$, the collected replay counts satisfy $
\widetilde N_{\ell,h}(s,a)=N_{\ell,h}(s,a)$ for  any $(h,s,a)\in\cR_\ell$ and any $\ell \in [L]$.
\end{lemma}
Let $\cF_\ell$ denote the history immediately before the replay stage of phase $\ell$.
For a fixed policy $\pi$, $\underline{V}^{\pi}$ provides a conservative estimate of its value function $V^{\pi}$.
Since the replay samples are independent of the candidate policies $\{\pi^{k,m}\}_{k\in[K],\,m\in[M]}$ conditional on $\cF_\ell$, we obtain the following guarantee:
\begin{lemma}
\label{lemma:pessimism}
For every $\ell\in[L]$, it holds that $
    \mathbb{E}\sqbk{\exp\bracket{\frac{\hat{v}_{\ell, m}^{-} - v_{\ell,m}}{H}}|\cF_{\ell}} \leq 1$.
Moreover, with probability at least $1-\delta/64$, it holds that $\sum_{\ell=1}^L(\hat v_{\ell,m}^--v_{\ell,m})
\le H\log\frac{64M}{\delta}$
 simultaneously for all $m\in[M]$,
\end{lemma}

We next bound the mean underestimation error.
\begin{lemma}[Conservative weighted bound]
\label{lemma:pessimistic_average_counterfactual}
With probability at least $1-SAHL\delta''$, simultaneously for all $\ell\in[L]$, $
    0\le\sum_{m=1}^M p_\ell(m)\bracket{v_{\ell,m}-\mathbb E[\hat v_{\ell,m}^-\mid\cF_\ell]}
\lesssim\frac{SAH^2}{K}\log\frac1{\delta''}.$
\end{lemma}

We defer the proofs to Appendices~\ref{app:subsubsec:replay_samples},~\ref{app:subsub:proof_of_pessimism}, and~\ref{app:subsec:proof_of_pessimistic_average_counterfactual}.


 \subsection{Gap-based Multiplicative Weights Update}

Define $W_{\ell} = \frac{1}{M}\sum_{m=1}^M \exp \left(\frac{1}{H}\sum_{\ell'=1}^{\ell-1}\left(\hat{v}^+_{\ell',m} - \hat{v}^{-}_{\ell', m}  \right) \right)$ for each $1\leq \ell \leq L+1$. By the multiplicative weights update rule \cref{eq:updatep}, we have that  $
W_{\ell+1}/W_{\ell} = \sum_{m=1}^M p_{\ell}(m)\exp \left( (\hat{v}^+_{\ell,m}  - \hat{v}^{-}_{\ell,m})/H \right).$

Define $(x)_+ = \max\{x,0\}$. 
With Lemma~\ref{lemma:pessimism}, we can show that 
 \begin{align}
\mathbb{E}\left[ (W_{\ell+1}/W_{\ell}) |\mathcal{F}_{\ell} \right] \leq 1+ \frac{4}{H}\cdot \sum_{m=1}^M p_{\ell}(m)\mathbb{E}\left[ (\hat{v}_{\ell,m}^{+} - v_{\ell,m} )_+ + v_{\ell,m} - \hat{v}_{\ell,m}^- |\mathcal{F}_{\ell} \right],
 \end{align}
which further implies
\begin{align}
\mathbb{E}\left[W_{L+1} \exp \left(-\frac{4}{H}\sum_{\ell=1}^L \sum_{m=1}^M p_{\ell}(m) \left[ (\hat{v}_{\ell,m}^+ - v_{\ell,m})_+ + v_{\ell,m} - \mathbb{E}[\hat{v}_{\ell,m}^- |\mathcal{F}_{\ell}] \right] \right) \right] \leq 1.
\end{align}
As a result, with high probability, 
we have that 
\begin{align}
\max_{m\in [M]}\frac{1}{L}\sum_{\ell=1}^L \left( \hat{v}^+_{\ell,m} - \hat{v}^-_{\ell,m}\right)\leq \frac{4}{L}\sum_{\ell=1}^L \sum_{m=1}^M p_{\ell}(m) \left[   (\hat{v}_{\ell,m}^+ - v_{\ell,m})_+ + v_{\ell,m} - \mathbb{E}[\hat{v}_{\ell,m}^- |\mathcal{F}_{\ell}]  \right] + O(\epsilon).\nonumber
\end{align}
Finally, on the intersection of the optimism and concentration events,
Lemmas~\ref{lemma:optimistic_average_counterfactual}
and~\ref{lemma:pessimistic_average_counterfactual} imply, respectively,
\begin{align*}
\frac{1}{L}\sum_{\ell=1}^L\sum_{m=1}^M
p_\ell(m)\left(\hat{v}^+_{\ell,m}-v_{\ell,m}\right)_+= O(\epsilon),\quad 
\frac{1}{L}\sum_{\ell=1}^L\sum_{m=1}^M
p_\ell(m)\left(
v_{\ell,m}
-\mathbb{E}\!\left[
\hat{v}^-_{\ell,m}\mid\mathcal{F}_\ell
\right]\right)
= O(\epsilon),
\end{align*}
under the choices of $K$ and $L$.
Combining these bounds with the preceding inequalities finishes
the proof.

\begin{remark}
We retain the term $(\hat{v}_{\ell,m}^+-v_{\ell,m})_+$
rather than $\hat{v}_{\ell,m}^+-v_{\ell,m}$
to ensure that the exponential-moment bound remains valid even when
optimism fails. Although $\hat{v}_{\ell,m}^+-v_{\ell,m}\ge 0$ holds
with high probability, we must still account for the failure event
when taking expectations.
\end{remark}

\section{Numerical Experiments}\label{sec:numerical_experiment}

We conduct two experiments to evaluate the episode efficiency of Algorithm~\ref{alg:main} and four reward-agnostic baseline algorithms in~\citet{jin2020reward, zhang2020taskagnostic, ridel2026improvedboundsrewardagnosticrewardfree, li2026minimax} on random nonstationary finite-horizon MDPs with $S=10$, $A=5$, $H=6$; the complete transition kernel, reward functions, and experimental settings are provided in Appendix~\ref{app:sec:numerical_experiment}. For each run $t$, total episode budget $B$, and $M$ reward models, we measure the worst-reward policy error and report its average over $T$ independent runs:
\[
\mathrm{Error}_t(B,M)
:= \max_{m\in[M]}
\bracket{V_{0,t}^{\star,m}
-
V_{0,t}^{\overline{\pi}^{m}_{B,t}, m}}, \quad 
\overline{\mathrm{Error}}_T(B,M)
:= \frac{1}{T}\sum_{t=1}^{T}\mathrm{Error}_t(B,M),
\]
where $\overline{\pi}^{m}_{B,t}$ is the policy obtained in the $t$-th run with an episode budget of $B$. The first experiment fixes $M = 16$ and examines the mean worst-reward error as the episode budget increases. The second varies $M$ and estimates the budget required to achieve $\overline {\mathrm{Error}}_8(B,M)\le0.1$.

\begin{figure}[ht]
    \centering
    \begin{subfigure}[t]{0.48\textwidth}
        \includegraphics[width=\linewidth]{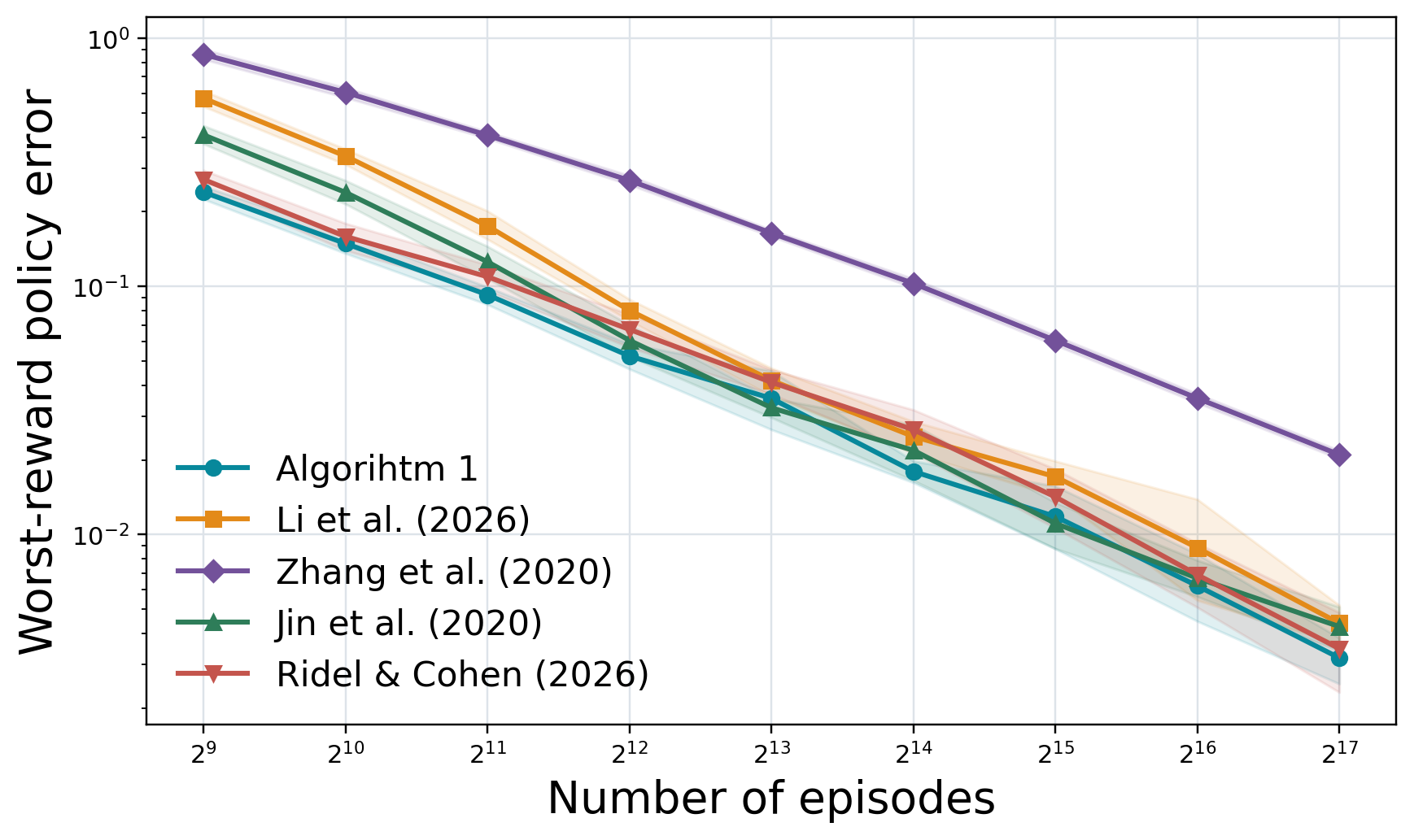}
        \caption{Convergence at $M=16$}
        \label{fig:episode_error}
    \end{subfigure}
    \hfill
    \begin{subfigure}[t]{0.48\textwidth}
        \centering
        \includegraphics[width=\linewidth]{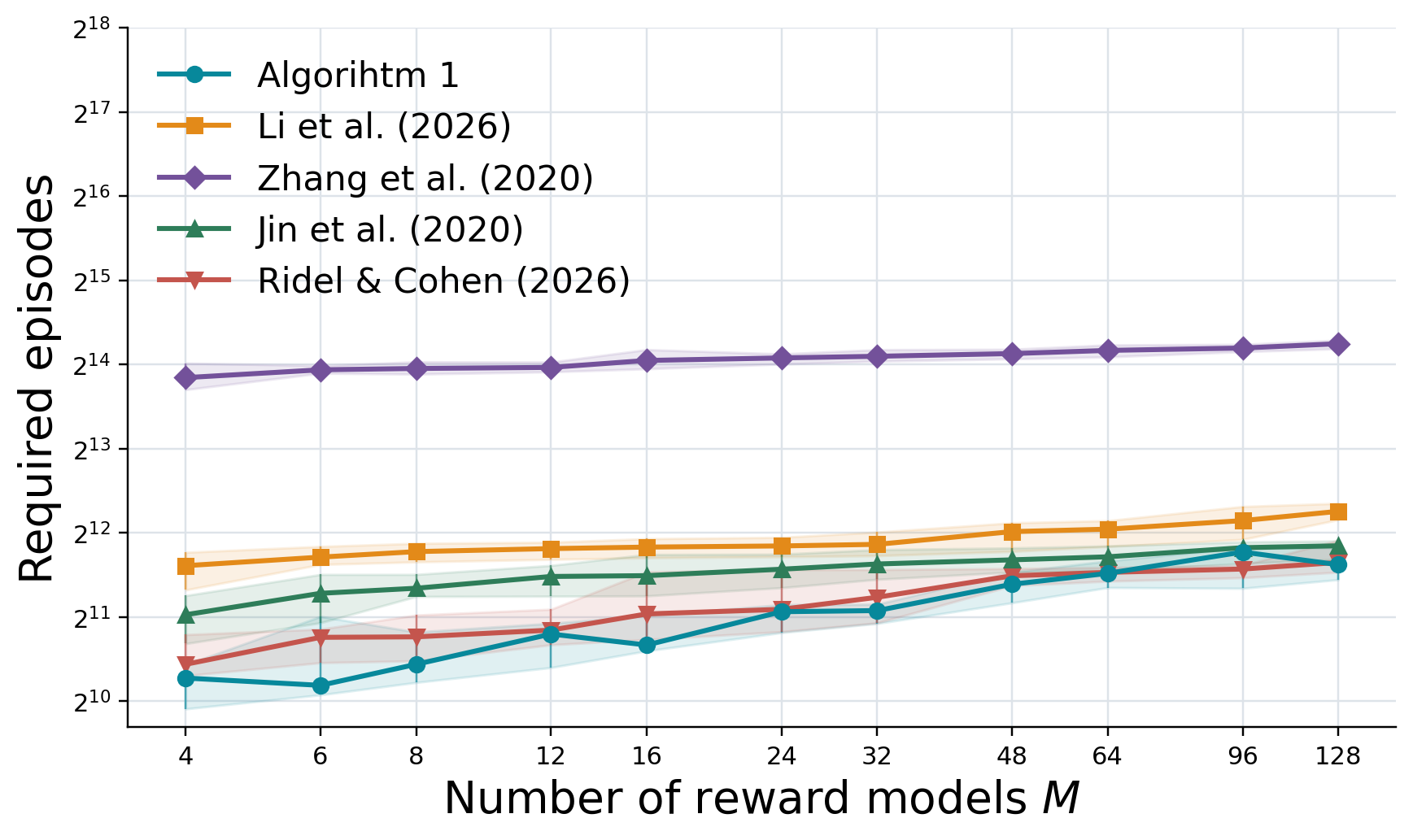}
        \caption{Reward-count scaling at $\epsilon=0.1$}
        \label{fig:reward_count}
    \end{subfigure}
    \caption{Empirical episode efficiency in multi-reward reinforcement learning. (a) Mean worst-reward policy error versus total episode budget at \(M=16\), averaged over 12 independent runs. (b) Estimated episode budget required to reach a mean worst-reward error of \(0.1\) as \(M\) varies, using 8 independent runs. Shaded regions indicate pointwise 95\% confidence intervals.}
    \label{fig:experiment}
\end{figure}

Figure~\ref{fig:episode_error} shows that our multi-reward algorithm achieves the lowest mean worst-reward error at $7$ of the $9$ tested budgets and first reaches the target error of \(0.1\) at \(2{,}048\) episodes, earlier than the four baseline implementations on the tested grid. All environment interactions are included in the reported budgets. \citet{jin2020reward} attains lower point estimates at \(8{,}192\) and \(32{,}768\) episodes.

Figure~\ref{fig:reward_count} shows that our method requires the lowest estimated budget at $10$ of the $11$ tested reward counts, with \citet{ridel2026improvedboundsrewardagnosticrewardfree} performing better at $M=96$. Increasing $M$ from $4$ to $128$ ($32\times$) raises our estimated budget from approximately $1{,}240$ to $3{,}153$ episodes ($2.54\times$). These results indicate favorable empirical episode efficiency and modest budget growth under the tested configurations.

\section{Discussion}
\label{sec:discussion}
In this paper, we study reinforcement learning with $M$ different rewards and propose an algorithm with optimal sample complexity of $\widetilde{O}(SAH^3\log (M)/\epsilon^2)$. This sample complexity bound is minimax optimal and contains no burn-in term. We leave the following interesting directions for future work: 
(i) whether the minimax-optimal sample complexity for reward-agnostic
reinforcement learning can be achieved without any burn-in terms; (ii) whether the shared structure among reward functions can be
exploited to obtain sharper sample-complexity bounds that adapt to
the intrinsic complexity of the reward class, rather than its
cardinality alone.

\newpage






\bibliographystyle{plainnat}
\bibliography{reference}


 \newpage
\appendixpage
\appendix

\section{Notations and Technical Lemmas}

\subsection{Notations}
\begin{table}[ht]
    \centering
    \caption{Parameters of MDP}
    \begin{tabular}{|>{\centering\arraybackslash}m{0.23\textwidth}|>{\centering\arraybackslash}m{0.67\textwidth}|}
    \hline
    \(\mathcal S,\; S=|\mathcal S|\) & State space and its size \\[2pt]
    \(\mathcal A,\; A=|\mathcal A|\) & Action space and its size \\[2pt]
    \(H\) & Horizon length \\[2pt]
    \(K\) & Number of learning episodes per phase \\[2pt]
    $L$ & Number of phases \\[2pt]
    $M$ & Number of reward functions \\ [2pt]
    \(s,s'\in\mathcal S\) & States \\[2pt]
    \(a,a'\in\mathcal A\) & Actions \\[2pt]
    \(h\in[H],\; h'\in[H]\) & Stage indices \\[2pt]
    \(k\in[K],\; k'\in[K]\) & Episode indices \\[2pt]
    \(P_{h,s,a}\) & Transition kernel at stage \(h\) \\[2pt]
    $\widehat P_{h,s,a}^{k}$ & Empirical transition kernel at $k$-th episode \\[2pt]
    $\widetilde{P}_{h,s,a}$ & Sealed replay transition kernel \\[2pt]
    \(r_h^m(s,a)\) & Reward at triple $(h,s,a)$ for $m\in [M]$ \\[2pt]
    \(\mu\in\Delta(\mathcal S)\) & Initial-state distribution \\[2pt]
    $\V(P, X)$ & Variance on random variable $X$ on $P$  \\ [2pt]
    $\langle f, g \rangle$ & inner product of $f, g$ \\ [2pt]
    $x\wedge y$ & $x\wedge y = \min \set{x, y}$ \\ [2pt]
    $x\vee y $ & $x\vee y = \max\set{x, y}$ \\ [2pt]
    $b_h^{k, m}(s,a)$ & Optimistic bonus at $k$-th episode for reward $r^{m}$. \\[2pt]
    $V_h^{\pi,m}(s),\;Q_h^{\pi,m}(s,a)$ & Value and action-value under policy $\pi$ for reward $r^m$ \\[2pt]
    $V_h^{*,m}(s),\;Q_h^{*,m}(s,a)$ & Optimal value and action-value for reward $r^m$ \\[2pt]
    $\overline{V}_{h}^{\ell, k,m}$ & $\overline{V}_{h}^{\ell,k,m}(s):= \max_{a\in \cA}\overline{Q}_h^{\ell, k, m}(s,a)$ \\[2pt]
    $\underline{V}_{h}^{\ell, k,m}$ & $\underline{V}_{h}^{\ell, k,m}(s):= \sum_{a\in \cA}\pi_h^{\ell, k, m}(a|s)\cdot \underline{Q}_h^{\ell, k, m}(s,a)$ \\[2pt]
    $\pi_{h}^{\ell, k, m}$ & $\pi_h^{\ell, k, m}(s)\in \arg\max_{a\in \cA}\overline{Q}_{h}^{\ell, k, m}(s,a)$, counterfactual policy at $\ell$-th phase, $k$-th episode \\ [2pt]
    $\pi^{\ell, k}$ & $\pi^{\ell, k} = \pi^{\ell, k, m^{\ell, k}}$, roll-out policy at $\ell$-th phase, $k$-th episode \\ [2pt]
    $N_{h}^{\ell, k}(s,a)$ & Number of samples in the most recently completed learning stage before episode \(k\) of phase \(\ell\), used to construct $\widehat{P}_{h,s,a}^{\ell, k}$. \\ [2pt]
    $N_{\ell, h}(s,a)$ & Cumulative number of samples collected in learning stage for $(h,s,a)$ within phase $\ell$. \\ [2pt]
    $\widetilde{N}_{\ell, h}(s,a)$ & Number of samples collected for $(h,s,a)$ during the phase $\ell$. \\ [2pt]
    $N_{h}^{k}(s,a)$ & Number of samples used to construct $\widehat{P}_{h,s,a}^{k}$ before $k$-th episode\\[2pt]
    $\cH_{k-1}^{\ell}$ & The $\sigma$-field generated by the history prior to sampling $m^{\ell,k}$ and $s_1^{\ell,k}$ in episode $k$ of phase $\ell$. \\[2pt]
   $\cA_{k}^{\ell}$ & $\cH_{k-1}^{\ell}\vee \sigma(s_{1}^{\ell, k})$ \\[2pt]
   $\cB_{k}^{\ell}$ & $\cH_{k-1}^{\ell}\vee \sigma(m^{\ell, k})$ \\[2pt]
   $\cF_{\ell}$& The filtration before the sealed replay stage starts at phase $\ell$ \\[2pt]
\hline
    \end{tabular}
    \label{tab:parameters_ofMDP}
\end{table}


\clearpage

\subsection{Technical Lemmas}

%


\begin{lemma}[Bennett's inequality]\label{bennet}
Let $Z,Z_1,...,Z_n$  be i.i.d.~random variables with values in $[0,1]$ and let $0<\delta<1$. Define $\mathbb{V}Z = \mathbb{E}\left[(Z-\mathbb{E}Z)^2 \right]$. Then one has
\begin{align}
\mathbb{P}\left[ \left|\mathbb{E}\left[Z\right]-\frac{1}{n}\sum_{i=1}^n Z_i  \right| > \sqrt{\frac{  2\mathbb{V}Z \log(2/\delta)}{n}} +\frac{\log(2/\delta)}{n} \right]\leq \delta.\nonumber
\end{align}
\end{lemma}
%
%
\begin{lemma}[Theorem 4 in  \cite{maurer2009empirical}]\label{empirical bernstein}
Consider any $0<\delta<1$ and any integer $n\geq 2$. 
Let $Z,Z_1,...,Z_n$  be a collection of i.i.d.~random variables falling within $[0,1]$. 
Define the empirical mean $\overline{Z} \coloneqq \frac{1}{n}\sum_{i=1}^n Z_{i}$ and empirical variance $\widehat{V}_n  \coloneqq \frac{1}{n}\sum_{i=1}^n (Z_i- \overline{Z})^2$. Then we have
\begin{align}
\mathbb{P}\left[ \left|\mathbb{E}\left[Z\right]-\frac{1}{n}\sum_{i=1}^n Z_i  \right| > \sqrt{\frac{  2\widehat{V}_n \log(4/\delta)}{n-1}} +\frac{7\log(4/\delta)}{3(n-1)} \right] \leq \delta.\nonumber
\end{align}
The claim follows by applying the one-sided bound to $Z$ and $1-Z$, each with failure probability $\delta/2$, and expressing the unbiased sample variance as $n\widehat V_n/(n-1)$.
\end{lemma}

\begin{lemma}\label{lemma:con}
Let $X_1,X_2,\ldots$ be a sequence of random variables taking values in $[0,l]$. 
For any $k\geq 1$, let $\mathcal{F}_k$ be the $\sigma$-algebra generated by $(X_1,X_2,\ldots,X_k)$, and define 
	$Y_k \coloneqq \mathbb{E}[X_k \mid \mathcal{F}_{k-1}]$. Then for any $\delta>0$, we have 
\begin{align}
& \mathbb{P}\left[ \exists n, \sum_{k=1}^n X_k \geq  2\sum_{k=1}^n Y_k+ 2l\log\frac{1}{\delta}\right]\leq \delta\nonumber
\\  & \mathbb{P}\left[  \exists n,  \sum_{k=1}^n Y_k \geq 2\sum_{k=1}^n X_k + 2l\log\frac{1}{\delta}  \right]    \leq \delta .\nonumber 
\end{align}
\end{lemma}
\begin{lemma}[Freedman's inequality]\label{lemma:self-norm}
	Let $(M_n)_{n\geq 0}$ be a martingale with respect to a filtration $(\mathcal F_n)_{n\geq0}$ such that $M_0=0$ and $|M_n-M_{n-1}|\leq c$  $(\forall n\geq 1)$ 
	hold for some quantity $c>0$. 
	Define $\mathsf{Var}_{n} \coloneqq \sum_{k=1}^n \mathbb{E}\left[  (M_{k}-M_{k-1})^2 \mid \mathcal{F}_{k-1}\right]$ for every $n\geq 0$. Then for any integer $n\geq 1$, $0<\epsilon\leq nc^2$, and $0<\delta<1$, one has 
\begin{align}
	\mathbb{P} \left[       |M_n|\geq 2\sqrt{2}\sqrt{\mathsf{Var}_n \log\frac{1}{\delta} } +2\sqrt{\epsilon \log\frac{1}{\delta} } +2c\log\frac{1}{\delta} \right]\leq 2\left(\log_2\left(\frac{nc^2}{\epsilon}\right) +1 \right)\delta.\nonumber
\end{align}
\end{lemma}

\begin{lemma}[Lemma 14 in~\cite{zhang2020reinforcement}]\label{lemma:mono}
Let $f: \Delta^{S} \times \mathbb{R}^S \times \mathbb{R} \times \mathbb{R} \rightarrow \mathbb{R}$ with $f(p,v,n,\iota) =pv+ \max\left\{\bar{c}_1\sqrt{\frac{ \mathbb{ V}(p,v) \iota }{n }} ,\bar{c}_2\frac{\iota}{n} \right\}$ with $\bar{c}_1= \frac{20}{3}$ and $\bar{c}_2 = \frac{400}{9}$.
Then  $f$ satisfies
\begin{enumerate}
\item $f(p,v,n,\iota)$ is non-decreasing in $v(s)$  for all $p\in \Delta^{S}$,$v\in[0,1]^S$  and $n,\iota>0$;
\item $f(p,v,n,\iota)\geq pv +  2\sqrt{\frac{ \mathbb{ V}(p,v) \iota }{n }} +\frac{14\iota}{3n}$ for all $p,v$ and $n,\iota>0$. 
\end{enumerate}
\end{lemma}

\begin{lemma}\label{lemma:doubling}
Let $N_h^k(s,a)$ denote the size of the most recently completed batch in Algorithm~\ref{alg:mvp} before episode $k$, with value zero if no batch has been completed. For $K\geq 1$, it holds that
\begin{align}
\sum_{k=1}^K \sum_{h=1}^H \frac{1}{\max\{ N_h^k(s_h^k,a_h^k),1\}}\leq 2SAH\log_2(2K)
\end{align}
\end{lemma}

\begin{lemma}[Ville's inequality]\label{lemma:ville}
	Let $X_0, X_1, \cdots$ be a non-negative supermartingale. Then for any real number $a > 0$,
	\begin{equation}
		\mathbb{P}\sqbk{\sup_{n\geq 0} X_n\geq a} \leq \frac{\mathbb{E}[X_0]}{a}
	\end{equation}
\end{lemma}

\clearpage
\section{Auxiliary Algorithms}
\label{sec:app:algcode}
In this section, we include Algorithm~\ref{alg:replay}: \texttt{Sealed-Replay} and~\ref{alg:pessimistic_evaluation}: \texttt{Conservative-Evaluation} for completeness.
{\small
\begin{algorithm}[htbp]
\caption{$\texttt{Sealed-Replay}(\cR, N, \cP)$}
\label{alg:replay}
\begin{algorithmic}[1]
{\small
\State {\textbf{input:} tuple set $\cR$, fixed sample size $N=\{N_{h}(s,a)\}$, indexed policy sequence $\cP=(\pi^1, \ldots, \pi^{K})$}
\State {\textbf{initialization:} set $\widetilde{N}_{h}(s,a), \widetilde{N}_{h}(s,a,\cdot)\gets 0$, initialize $\widetilde P_{h,s,a}$ to the fixed uniform distribution $\mathrm{Unif}(\cS)$ for all $(h,s,a)\in\cR$; and set $L_{\mathrm{rep}}=8$.}
\For{$l= 1,2,\cdots, L_{\mathrm{rep}}$}
\For{$k=1,2,\cdots,K$}
\State{Start a fresh episode with $\widetilde{s}_1\sim\mu$.}
\For{$h=1,2,\cdots, H$}
\State{Observe $\widetilde{s}_{h}$, execute $\widetilde{a}_{h} = \pi^{k}_{h}(\widetilde{s}_{h})$, observe $\widetilde{s}_{h+1}$.}
\If{$(h,\widetilde{s}_{h}, \widetilde{a}_{h})\in \cR$ and $\widetilde{N}_{h}(\widetilde{s}_{h}, \widetilde{a}_{h})< N_{h}(\widetilde{s}_{h}, \widetilde{a}_{h})$}
\State{Update $\widetilde{N}_{h}(\widetilde{s}_{h}, \widetilde{a}_{h}, \widetilde{s}_{h+1})\gets \widetilde{N}_{h}(\widetilde{s}_{h}, \widetilde{a}_{h}, \widetilde{s}_{h+1}) + 1$, $\widetilde{N}_{h}(\widetilde{s}_{h}, \widetilde{a}_{h})\gets \widetilde{N}_{h}(\widetilde{s}_{h}, \widetilde{a}_{h}) + 1$}
\State{Update retained empirical kernel with $\widetilde{P}_{h,\widetilde{s}_{h}, \widetilde{a}_{h},s'} = \frac{\widetilde{N}_{h}(\widetilde{s}_{h}, \widetilde{a}_{h}, s')}{\widetilde{N}_{h}(\widetilde{s}_{h}, \widetilde{a}_{h})}$ for all $s'\in \cS$}
\EndIf
\EndFor
\EndFor
\EndFor
}
\State{\textbf{return: }$\widetilde{P}$, $\widetilde{N}$}
\end{algorithmic}
\end{algorithm}
}

{\small
\begin{algorithm}[ht]
\caption{$\texttt{Conservative-Evaluation}(P, N, \cR, \pi, r)$}
\label{alg:pessimistic_evaluation}
\begin{algorithmic}[1]
{\small
\State {\textbf{input:} Evaluation kernel $P$, count function $N$, retained tuple set $\cR$, policy $\pi$, reward function $r$}
\State {\textbf{initialization:} Set $\underline{V}_{H+1}(s)\gets 0$ for all $s\in \cS$}
\For{$h=H, H-1, \ldots, 1$}{\color{blue}\Comment{evaluate $\pi$ under $r$}}
\For{$(s,a)\in \cS\times \cA$}
\State{\[
\underline{Q}_{h}(s,a)\gets \begin{cases}
                \max\set{r_h(s,a)+\dfrac{N_{h}(s,a)}{N_{h}(s,a) + 1}\langle P_{h,s,a},\underline V_{h+1}\rangle,0}, &\substack{(h,s,a)\in\cR},\\
                r_{h}(s,a),&\text{otherwise}.
            \end{cases}
			\]}
\EndFor
\State{For all $s\in\cS$, set $\underline{V}_{h}(s)\gets\sum_{a\in\cA}\pi_h(a|s)\underline{Q}_{h}(s,a)$.}
\EndFor
}
\State{\textbf{return:} $\underline{V}_{1}$}
\end{algorithmic}
\end{algorithm}
}

\section{Missing Proofs in Section~\ref{sec:minimax}}\label{app:minimax_additional_proof}

In this section, we provide the full proof of Theorem~\ref{thm:adv_sample_complexity}.

\begin{theorem}[Restatement of
Theorem~\ref{thm:adv_sample_complexity}]
\label{thm:adv_sample_complexity_restatement}
There exists a universal constant $C>0$ such that the following holds.
Under the multi-reward MDP setting with $M\ge2$ and $SAH\ge2$,
fix any $\delta\in(0,1)$ and $\epsilon\in(0,H]$. Set
\[
\iota=\log^4\frac{SAH\log M}{\min\set{\epsilon,1}\delta},\qquad
K=\left\lceil C\frac{SAH^2\iota^{3/4}}{\epsilon}\right\rceil,\qquad
L=\left\lceil C\frac{H\iota^{1/4}}{\epsilon}\log M\right\rceil.
\]
Run Algorithm~\ref{alg:main} for $L$ phases with $K$ learning episodes per phase.
With probability at least $1-\delta$, the algorithm returns episode-level mixture policies
$\{\overline{\pi}^{m}\}_{m\in[M]}$ such that
\[
V_0^{*,m}
-
V_0^{\overline{\pi}^{m},m}
\le \epsilon
\qquad\text{for every }m\in[M],
\]
using a total of $9KL$ episodes, which is at most
\[
O\bracket{\frac{SAH^3}{\epsilon^2}(\log M)
\log^4\bracket{\frac{SAH\log M}{\min\set{\epsilon,1}\delta}}},
\]
including all learning and replay episodes.
\end{theorem}

\paragraph{Proof outline:} We first bound the weighted average optimistic error over all $KL$ learning episodes. Then, we control the cumulative overestimation and conditional mean bias of conservative evaluation. The reward-weight update combines these bounds to control each reward's cumulative certificate across the phases $
\sum_{\ell=1}^L(\hat v_{\ell,m}^+-\hat v_{\ell,m}^-).$ 

Finally, we use initial-state concentration to transfer the guarantee to the output mixture $\overline\pi^m$.

We retain $\delta'$ and $\delta''$ as the local confidence parameters, with $0<\delta'\leq\min\{e^{-1},(2KL)^{-1}\}$ and $0<\delta''\leq e^{-1}$. These ranges hold for the choices in Algorithm~\ref{alg:main}.

At phase $\ell$, recall that   $\cF_{k-1}^{\ell}$ contains all the information before the initial state $s_{1}^{\ell, k}$. I.e., $\cF_{k-1}^{\ell}$ includes all completed learning and replay history, reward-sampling distribution $p_{\ell}$, and the sampled reward index $m^{\ell,k}$, before observing $s_1^{\ell,k}$. Let $\cF_\ell$ contain the history immediately before phase $\ell$'s sealed replay stage, and recall
\[
v_{\ell,m}:=\frac1K\sum_{k=1}^K V_1^{\pi^{\ell,k,m},m}(s_1^{\ell,k}),\qquad
\overline\pi^m=\mathrm{Unif}\bigl((\pi^{\ell,k,m})_{\ell\in[L],\,k\in[K]}\bigr).
\]
Within a fixed phase, we suppress $\ell$ when there is no confusion. Every candidate $\pi^{\ell,k,m}$ is fixed before the independent draw $s_1^{\ell,k}\sim\mu$. Define
\[
\begin{aligned}
\Delta_{\mathrm{init}}^{m}
:={}&\frac{1}{KL}\sum_{\ell=1}^L\sum_{k=1}^K
\E_{s_1\sim\mu}\sqbk{V_1^{*,m}(s_1)-V_1^{\pi^{\ell,k,m},m}(s_1)}\\
&-\frac{1}{KL}\sum_{\ell=1}^L\sum_{k=1}^K
\bracket{V_1^{*,m}(s_1^{\ell,k})-V_1^{\pi^{\ell,k,m},m}(s_1^{\ell,k})}.
\end{aligned}
\]
For each $m\in[M]$, the output policy error decomposes as
\begin{equation}
\label{equ:sample_complexity_decomposition}
V_0^{*,m}-V_0^{\overline\pi^m,m}
=\underbrace{\Delta_{\mathrm{init}}^{m}}_{\mathrm{Term}_1}
+\underbrace{\frac1{KL}\sum_{\ell=1}^L\sum_{k=1}^K
\bracket{V_1^{*,m}(s_1^{\ell,k})-V_1^{\pi^{\ell,k,m},m}(s_1^{\ell,k})}}_{\mathrm{Term}_2}.
\end{equation}

\paragraph{Bound of $\mathrm{Term}_1$.} Since the counterfactual policy $\pi^{\ell, k, m}$ is fixed before drawing $s_{1}^{\ell, k}$, the functions $V_1^{*,m}$ and $V_1^{\pi^{\ell,k,m},m}$ are $\cF_{k-1}^{\ell}$-measurable and 
\[
0\leq V_1^{*,m}(s) - V_1^{\pi^{\ell,k,m},m}(s)\leq H\implies 0\leq \frac{V_1^{*,m}(s) - V_1^{\pi^{\ell,k,m},m}(s)}{H}\leq 1
\]
Since for $x\in [0,1]$, the convexity of exponential function gives 
\[
e^{-x}\leq (1-x)e^{0} +xe^{-1}=1-(1-e^{-1})x.
\]
Hence
\[
\exp\bracket{-\frac{V_1^{*,m}(s_{1}^{\ell, k}) - V_1^{\pi^{\ell,k,m},m}(s_{1}^{\ell, k})}{H}} \leq 1 - \frac{1-e^{-1}}{H}\bracket{V_1^{*,m}(s_{1}^{\ell, k}) - V_1^{\pi^{\ell,k,m},m}(s_{1}^{\ell, k})}.
\]
Taking the conditional expectation on both sides with $\cF_{k-1}^{\ell}$, recall that $\pi^{\ell, k, m}$ is fixed before the draw of $s_{1}^{\ell, k}$ and $V_1^{*,m}$ and $V_1^{\pi^{\ell,k,m},m}$ are $\cF_{k-1}^{\ell}$-measurable, then
\begin{align}
    &\E\sqbk{\exp\bracket{-\frac{V_1^{*,m}(s_{1}^{\ell, k}) - V_1^{\pi^{\ell,k,m},m}(s_{1}^{\ell, k})}{H}}\mid \cF_{k-1}^{\ell}}\nonumber\\
    \leq&1 - \frac{1-e^{-1}}{H}\bracket{V_{0}^{*, m} - V_{0}^{\pi^{\ell, k, m}, m}}\nonumber\\
    \leq&\exp\bracket{-\frac{1-e^{-1}}{H}\bracket{V_{0}^{*, m} - V_{0}^{\pi^{\ell, k, m},m}}}\label{equ:app:exponential_bound_by_exponential}.
\end{align}
where the last inequality follows $1-x\leq e^{-x}$ for all $x\in \mathbb{R}$. Multiplying both sides of~\cref{equ:app:exponential_bound_by_exponential} by the positive, $\cF_{k-1}^{\ell}$-measurable factor
\[
\exp\bracket{\frac{1-e^{-1}}{H}\bracket{V_{0}^{*, m} - V_{0}^{\pi^{\ell, k, m}, m}}}
\]
and bringing this factor inside the conditional expectation yields:
\begin{equation}
\label{equ:conditional_bound}
    \E\sqbk{\exp\set{\frac{1}{H}\sqbk{(1-e^{-1})\bracket{V_{0}^{*, m} - V_{0}^{\pi^{\ell, k, m}, m}} - \bracket{V_{1}^{*, m}(s_{1}^{\ell, k}) - V_1^{\pi^{\ell,k,m},m}(s_{1}^{\ell, k})}}}\mid \cF_{k-1}^{\ell}}\leq1
\end{equation}
holds for all $\ell, k\in [L]\times [K]$. Write the product formula
\begin{equation}
\begin{aligned} 
&\exp\left\{ \frac1H \sum_{\ell=1}^L\sum_{k=1}^K \Big[ (1-e^{-1}) \left( V_0^{*,m}-V_0^{\pi^{\ell,k,m},m} \right)  - \left( V_1^{*,m}(s_1^{\ell,k}) - V_1^{\pi^{\ell,k,m},m}(s_1^{\ell,k}) \right) \Big] \right\} \\ 
&= \prod_{\ell=1}^L\prod_{k=1}^K \exp\!\left\{ \frac1H \Big[ (1-e^{-1}) \left( V_0^{*,m}-V_0^{\pi^{\ell,k,m},m} \right) - \left( V_1^{*,m}(s_1^{\ell,k}) - V_1^{\pi^{\ell,k,m},m}(s_1^{\ell,k}) \right) \Big] \right\}. 
\end{aligned} 
\end{equation}
Denote
\begin{align*}
    z_{\ell, k} =& \exp\!\left\{ \frac1H \Big[ (1-e^{-1}) \left( V_0^{*,m}-V_0^{\pi^{\ell,k,m},m} \right) - \left( V_1^{*,m}(s_1^{\ell,k}) - V_1^{\pi^{\ell,k,m},m}(s_1^{\ell,k}) \right) \Big] \right\}\\
    Z_{\ell, k}=& \bracket{\prod_{\ell'=1}^{\ell-1}\prod_{k'=1}^{K}z_{\ell', k'}}\bracket{\prod_{k'=1}^{k}z_{\ell, k'}}
\end{align*}
i.e., $Z_{\ell, k}$ is the product of $z_{\ell', k'}$ till $\ell, k$. 
By the tower rule and~\cref{equ:conditional_bound}
\begin{align*}
    \E\sqbk{Z_{L, K}} =& \E\sqbk{Z_{L, K-1}z_{L, K}} =\E\sqbk{\E\sqbk{Z_{L, K-1}z_{L, K}\mid \cF_{K-1}^{L}}}\\
    =& \E\sqbk{Z_{L,K-1}\E\sqbk{z_{L, K}\mid \cF_{K-1}^{L}}}\leq \E\sqbk{Z_{L, K-1}}\leq 1
\end{align*}
Then, iteratively apply the above iteration and the definition of the mixture policy $\overline{\pi}^{m}=\mathrm{Unif}((\pi^{\ell, k, m})_{\ell\in [L], k\in [K]})$, $V^{\overline{\pi}^{m}, m}_{0} = \frac{1}{KL}\sum_{\ell, k} V_{0}^{\pi^{\ell, k, m}, m}$
\begin{align*}
    &\E\sqbk{Z_{L, K}} \\
    =&\mathbb E\!\left[ \exp\!\left\{ \frac1H \sum_{\ell=1}^L\sum_{k=1}^K \Big[ (1-e^{-1}) \left( V_0^{*,m}-V_0^{\pi^{\ell,k,m},m} \right) - \left( V_1^{*,m}(s_1^{\ell,k}) - V_1^{\pi^{\ell,k,m},m}(s_1^{\ell,k}) \right) \Big] \right\} \right]\\
    =&\mathbb E\!\left[ \exp\!\left\{ \frac{KL}{H} \Big[ (1-e^{-1}) \left( V_0^{*,m}-V_0^{\overline{\pi}^{m},m} \right) - \mathrm{Term}_{2} \Big] \right\} \right] \leq 1
\end{align*}
Since the exponential term is a nonnegative random variable with expectation at most $1$, Markov's inequality implies
\[
\mathbb{P}\bracket{\exp\set{\frac{KL}{H}\sqbk{(1-e^{-1})\bracket{V_{0}^{*, m} - V_{0}^{\overline{\pi}^{m}, m}} - \mathrm{Term_2}}}\geq \frac{1}{\delta'}} \leq \delta'
\]
Then for any fixed $m$, with probability at least $1-\delta'$
\[
(1-e^{-1})\bracket{V_{0}^{*, m} - V_{0}^{\overline{\pi}^{m}, m}}\leq \mathrm{Term}_{2} +\frac{H}{KL}\log\frac{1}{\delta'}
\]
combine $(1-e^{-1})^{-1}\leq 2$, it further implies
\[
\bracket{V_{0}^{*, m} - V_{0}^{\overline{\pi}^{m}, m}} \leq 2\mathrm{Term}_{2} +\frac{2H}{KL}\log \frac{1}{\delta'}
\]
simultaneously for all rewards with probability at least $1-M\delta'$.

\paragraph{Bound of $\mathrm{Term}_2$.} On the optimism event in Lemma~\ref{lemma:optimism},
\begin{equation}
\label{equation:certificate_decomposition}
\mathrm{Term}_2\leq\frac1L\sum_{\ell=1}^L(\hat v_{\ell,m}^+-v_{\ell,m})=\underbrace{\frac1L\sum_{\ell=1}^L(\hat v_{\ell,m}^+-\hat v_{\ell,m}^-)}_{\text{Cumulative certificate}}
+\underbrace{\frac1L\sum_{\ell=1}^L(\hat v_{\ell,m}^--v_{\ell,m})}_{\text{Conservative evaluation error}}.
\end{equation}
The conservative evaluation error is signed. Iterating the conditional exponential-moment bound in Lemma~\ref{lemma:pessimism} and applying Markov's inequality gives, with probability at least $1-M\delta'$, simultaneously for all $m\in[M]$,
\[
\sum_{\ell=1}^L(\hat v_{\ell,m}^--v_{\ell,m})\leq H\log\frac1{\delta'}.
\]
Intersecting with the initial-state and optimism events therefore yields
\[
V_0^{*,m}-V_0^{\overline\pi^m,m}
\leq\frac2L\sum_{\ell=1}^L(\hat v_{\ell,m}^+-\hat v_{\ell,m}^-)
+\frac{2H}{L}\log\frac1{\delta'}+\frac{2H}{KL}\log\frac1{\delta'}.
\]
It remains to control the cumulative certificates through the reward-weight update. The required weighted optimistic error and conditional mean evaluation bias are analyzed in Appendix~\ref{sec:app:optimistic_part} and Appendices~\ref{sec:app:pessimistic_part}--\ref{app:subsec:proof_of_pessimistic_average_counterfactual}, respectively, and combined in Appendix~\ref{app:proof_of_adv_sample_complexity}.

\subsection{Optimistic Evaluation Part for Multi-Reward MDPs}
\label{sec:app:optimistic_part}
We first show the key Lemmas in the multi-reward MDP setting with respect to the optimistic evaluation in the shared-model multi-reward learning.

\subsubsection{Proof of Lemma~\ref{lemma:optimism}}
\label{app:subsubsec:proof_of_optimism}
The proof follows the same structure as the optimism analysis in~\cite{zhang2024settling}, with additional concentration events for the $M$ reward functions over all $KL$ learning episodes. Write $\widehat P_{h,s,a}^{\ell,k}$ for the empirical transition and let $N_{h}^{\ell, k}(s,a)$ denote the value of $N_h(s,a)$ in Algorithm~\ref{alg:mvp} before episode $k$ of phase $\ell$, i.e., the number of samples used to construct the current empirical transition tuple $\widehat P_{h,s,a}^{\ell,k}$. Note that the shared empirical transition kernel $\widehat{P}_{h,s,a}^{\ell, k}$ is not reinitialized at phase boundaries, as well as $N_{h}^{\ell, k}(s,a)$, distinct from the sealed replay model $\widetilde{P}_{\ell}=\{\widetilde{P}_{\ell, h, s,a}\}_{h,s,a}$, and count $\widetilde{N}_{\ell} = \{\widetilde{N}_{\ell, h}(s,a)\}_{h,s,a}$ in line~\ref{line:alg:main:sealed_replay}.
\begin{lemma}[Formal statement of Lemma~\ref{lemma:optimism}]
    \label{lemma:app:optimism}
  With probability at least $1-4SAHKLM\delta'$, one has
  \[
  Q_{h}^{*, m}(s,a)\leq \overline{Q}_{h}^{\ell,k,m}(s,a)\quad \text{and}\quad V_{h}^{*, m}(s) \leq \overline{V}_{h}^{\ell,k,m}(s).
  \]
  for all $(h,s,a,\ell,k,m) \in [H]\times \cS\times \cA\times [L]\times [K]\times [M]$.
\end{lemma}
\begin{proof}
For each $(h,s,a)$, represent its successive learning transitions over all phases by an independent i.i.d. sequence with law $P_{h,s,a}$, independent of the replay streams. Since the learning state persists across phases, the potential batches are fixed blocks of this sequence, of sizes $1,1,2,4,\ldots$, whether or not they are actually completed. Apply Lemma~\ref{empirical bernstein}, with confidence parameter $4\delta'$, to $V_{h+1}^{*,m}/H$ on each block of size at least two and each $m\in[M]$. There are at most $1+\lfloor\log_2(KL)\rfloor\leq KL$ potential batches per row. Thus, with probability at least $1-4SAHM(1+\lfloor\log_2(KL)\rfloor)\delta'$, every empirical row used by the algorithm with $N_h^{\ell,k}(s,a)\geq2$ satisfies
\begin{align}
\left|\left\langle\widehat P_{h,s,a}^{\ell,k}-P_{h,s,a},V_{h+1}^{*,m}\right\rangle\right|
&\leq \sqrt{\frac{2\V(\widehat P_{h,s,a}^{\ell,k},V_{h+1}^{*,m})\log(1/\delta')}{N_h^{\ell,k}(s,a)-1}}
+\frac{7H\log(1/\delta')}{3(N_h^{\ell,k}(s,a)-1)}\nonumber\\
&\leq 2\sqrt{\frac{\V(\widehat P_{h,s,a}^{\ell,k},V_{h+1}^{*,m})\log(1/\delta')}{N_h^{\ell,k}(s,a)}}
+\frac{14H\log(1/\delta')}{3N_h^{\ell,k}(s,a)}.
\label{eq_lemma1_ref.5}
\end{align}
This event is uniform over all $KL$ learning episodes, including adaptively selected completed batches.

We prove optimism by backward induction on $h$. The terminal values are zero. If $N_h^{\ell,k}(s,a)<2$, the bonus and $\log(1/\delta')\geq1$ imply $\overline Q_h^{\ell,k,m}(s,a)=H\geq Q_h^{*,m}(s,a)$. Otherwise, apply Lemma~\ref{lemma:mono} to $\overline V_{h+1}^{\ell,k,m}/H$ and multiply by $H$. The algorithm's constants satisfy $c_1\geq20/3$ and $c_2\geq400/9$. Whenever the Q-update is not clipped at $H$, the monotonicity of $f$ and the induction hypothesis give
\begin{align}
\overline Q_h^{\ell,k,m}(s,a)
&\geq r_h^m(s,a)+H f\left(\widehat P_{h,s,a}^{\ell,k},\frac{\overline V_{h+1}^{\ell,k,m}}{H},N_h^{\ell,k}(s,a),\log\frac1{\delta'}\right)\nonumber\\
&\geq r_h^m(s,a)+H f\left(\widehat P_{h,s,a}^{\ell,k},\frac{V_{h+1}^{*,m}}{H},N_h^{\ell,k}(s,a),\log\frac1{\delta'}\right)\nonumber\\
&\geq r_h^m(s,a)+\langle P_{h,s,a},V_{h+1}^{*,m}\rangle
=Q_h^{*,m}(s,a),
\label{eq:Qhk-sa-LB1}
\end{align}
where the last inequality uses the second conclusion of Lemma~\ref{lemma:mono} and~\eqref{eq_lemma1_ref.5}. Clipping also preserves this bound since $Q_h^{*,m}(s,a)\leq H$. Taking the maximum over $a$ proves $\overline V_h^{\ell,k,m}(s)\geq V_h^{*,m}(s)$ on the same event.
\end{proof}

\subsubsection{Proof of Lemma~\ref{lemma:optimistic_average_counterfactual}}
\label{app:subsubsec:proof_of_optimistic_average_counterfactual}
The following cumulative regret bound guarantee extends the analysis in~\cite{zhang2024settling}, which controls the cumulative optimism surplus along the roll-out policies.
\begin{lemma}
    \label{lemma:reward_switching}
  Over all $KL$ learning episodes in Algorithm~\ref{alg:mvp}, with probability at least $1-5\delta'$,
  \[
  \sum_{\ell=1}^L\sum_{k=1}^K
  \big(\overline V_1^{\ell,k,m^{\ell,k}}(s_1^{\ell,k})-V_1^{\pi^{\ell,k,m^{\ell,k}},m^{\ell,k}}(s_1^{\ell,k})\big)\!\lesssim\!\min\big\{\sqrt{SAH^3KL\log_2^3(2KL)\log\frac M{\delta'}},HKL\big\}.
  \]
\end{lemma}
The proof of Lemma~\ref{lemma:reward_switching} is given in Appendix~\ref{app:sec:proof_of_reward_switching} and its supporting profile and bonus arguments, which apply once from initialization to all $KL$ episodes of the persistent learning process. For each fixed global profile, each optimistic continuation vector at stage $h+1$ depends only on the known rewards and learning streams at stages strictly greater than $h$. The profile concentration events are uniform over all profiles and candidate choices, so they also cover the reward sampling distribution $p_\ell$ determined by previous replay outcomes. The forward martingale arguments use histories containing these outcomes; the next learning transition still has conditional law $P_{h,s,a}$. Thus the same $5\delta'$ failure bound applies to the entire run.

In proving Lemma~\ref{lemma:optimistic_average_counterfactual}, denote the optimism gap and weighted gap
\[
\Delta^{\ell,k,m}:=\overline V_1^{\ell,k,m}(s_1^{\ell,k})-V_1^{\pi^{\ell,k,m},m}(s_1^{\ell,k}),\qquad
\Delta^{\ell,k}:=\sum_{m=1}^M p_\ell(m)\Delta^{\ell,k,m},
\]
and the nonnegative counterfactual differences and their weighted average by
\[
\Delta_+^{\ell,k,m}:=\sqbk{\Delta^{\ell,k,m}}_+,\qquad
\Delta_+^{\ell,k}:=\sum_{m=1}^M p_\ell(m)\Delta_+^{\ell,k,m}.
\]
Let $\cH_{k-1}^{\ell}$ contain all completed learning and replay history, including prior auxiliary evaluations and weight updates, before drawing $m^{\ell,k}$ and $s_1^{\ell,k}$. The distribution $p_\ell$, current counterfactual policies $\{\pi^{\ell, k, m}\}_{m\in [M]}$, and optimistic values are measurable with respect to this history. Conditional on $\cH_{k-1}^{\ell}$, the fresh draws $s_1^{\ell,k}\sim\mu$ and $m^{\ell,k}\sim p_\ell$ are independent. Use the conditioning fields
\begin{equation}
\label{equ:sigma_fields}
    \cA_k^{\ell}:=\cH_{k-1}^{\ell}\vee\sigma(s_1^{\ell,k}),\qquad\cB_k^{\ell}:=\cH_{k-1}^{\ell}\vee\sigma(m^{\ell,k}).
\end{equation}
Thus $\cA_k^{\ell}$ reveals the initial state but not the reward index, whereas $\cB_k^{\ell}$ reveals the reward index but not the initial state. In particular,
\[
\E\sqbk{\Delta_+^{\ell,k,m^{\ell,k}}\mid\cA_k^{\ell}}
=\sum_{m=1}^M p_\ell(m)\Delta_+^{\ell,k,m}
=\Delta_+^{\ell,k}.
\]
The selected counterfactual difference $\Delta_+^{\ell,k,m^{\ell,k}}$ is an unbiased random probe of the weighted average $\Delta_+^{\ell,k}$. Consider the random process with $M_0^1=1$, $M_0^{\ell}=M_K^{\ell-1}$ for $\ell\geq2$, and
\[
M_k^{\ell}:=M_{k-1}^{\ell}\exp\bracket{(1-e^{-1})\frac{\Delta_+^{\ell,k}}{H}-\frac{\Delta_+^{\ell,k,m^{\ell,k}}}{H}},\qquad k\in[K].
\]
Since $\Delta_+^{\ell,k}$ is $\cA_k^{\ell}$-measurable,
\[
\E\sqbk{M_k^{\ell}\mid\cA_k^{\ell}}
=M_{k-1}^{\ell}\exp\bracket{(1-e^{-1})\frac{\Delta_+^{\ell,k}}{H}}
\E\sqbk{\exp\bracket{-\frac{\Delta_+^{\ell,k,m^{\ell,k}}}{H}}\mid\cA_k^{\ell}}.
\]
For all $y\in[0,1]$, $e^{-y}\leq1-(1-e^{-1})y$. Since $0\leq\Delta_+^{\ell,k,m}\leq H$, we have
\[
\E\sqbk{\exp\bracket{-\frac{\Delta_+^{\ell,k,m^{\ell,k}}}{H}}\mid\cA_k^{\ell}}
\leq1-(1-e^{-1})\frac{\Delta_+^{\ell,k}}{H}
\leq\exp\bracket{-(1-e^{-1})\frac{\Delta_+^{\ell,k}}{H}}.
\]
Therefore, $\E[M_k^{\ell}\mid\cA_k^{\ell}]\leq M_{k-1}^{\ell}$. Taking conditional expectations over the initial distribution again gives $\E[M_k^{\ell}\mid\cH_{k-1}^{\ell}]\leq M_{k-1}^{\ell}$. By Ville's inequality in Lemma~\ref{lemma:ville}, for any concentration failure level $\delta''\in(0,1)$,
\[
\mathbb P\bracket{\max_{\ell\in[L],\,k\in[K]}M_k^{\ell}\geq\frac1{\delta''}}\leq\delta''.
\]
Consequently, with probability at least $1-\delta''$,
\[
\sum_{\ell=1}^L\sum_{k=1}^K\Delta_+^{\ell,k}
\leq\frac1{1-e^{-1}}\bracket{\sum_{\ell=1}^L\sum_{k=1}^K\Delta_+^{\ell,k,m^{\ell,k}}+H\log\frac1{\delta''}}
\leq2\bracket{\sum_{\ell=1}^L\sum_{k=1}^K\Delta_+^{\ell,k,m^{\ell,k}}+H\log\frac1{\delta''}}.
\]
The preceding concentration event is unconditional. At the same time, on its intersection with the optimism event in Lemma~\ref{lemma:app:optimism}, the optimism implies
\[
\Delta_+^{\ell,k,m}=\Delta^{\ell,k,m},\qquad \Delta_+^{\ell,k}=\Delta^{\ell,k}
\]
Intersecting further with Lemma~\ref{lemma:reward_switching} yields, with probability at least $1-(4SAHKLM+5)\delta'-\delta''$,
\begin{align}
\frac1L\sum_{\ell=1}^L\sum_{m=1}^M p_\ell(m)(\hat v_{\ell,m}^+-v_{\ell,m})
&=\frac1{KL}\sum_{\ell=1}^L\sum_{k=1}^K\Delta^{\ell,k}\nonumber\\
&\leq\frac2{KL}\bracket{\sum_{\ell=1}^L\sum_{k=1}^K\Delta^{\ell,k,m^{\ell,k}}+H\log\frac1{\delta''}}\nonumber\\
&\lesssim\sqrt{\frac{SAH^3}{KL}\log_2^3(2KL)\log\frac M{\delta'}}
+\frac H{KL}\log\frac1{\delta''}.
\label{equ:optimism_average_counterfactual_error}
\end{align}
The concentration failure level in this argument is free. Applying Ville's inequality at level $\delta/64$ gives the probability and bound stated in Lemma~\ref{lemma:optimistic_average_counterfactual}, with the algorithm's replay confidence parameter unchanged. The quantity $v_{\ell,m}$ already uses the recorded initial states $(s_{1}^{\ell, k})_{k\in [K]}$, so this step requires no additional initial-state correction. This completes the proof.



\subsection{Conservative Evaluation and Multiplicative Weight Update Parts for Multi-Reward MDPs}
\label{sec:app:pessimistic_part}
In this section, we prove the replay and exponential evaluation guarantees. Retain the auxiliary run and histories defined above, and write
\begin{align}
d_{\ell,h}^{k,m}(s,a)
&:=\mathbb P_{\pi^{\ell,k,m}}\bracket{(s_h,a_h)=(s,a)\mid s_1=s_1^{\ell,k}},\label{equ:occupancy}\\
D_{\ell,h,s,a}^{\mathrm{avg}}
&:=\sum_{m=1}^M p_\ell(m)\sum_{k=1}^K d_{\ell,h}^{k,m}(s,a).\label{equ:average_cumulative_occupancy}
\end{align}
For the recorded roll-out policies $\pi^{\ell,k}:=\pi^{\ell,k,m^{\ell,k}}$, also define
\begin{align}
    d_h^{\pi,\mu}(s,a):=&\mathbb P_{\pi,s_1\sim\mu}\bracket{(s_h,a_h)=(s,a)}\\
    D_{\ell,h,s,a}^{\mu}:=&\sum_{k=1}^K d_h^{\pi^{\ell,k},\mu}(s,a).
\end{align}
Here $N_{\ell,h}(s,a)$ counts visits in phase $\ell$'s $K$ learning episodes.

\subsubsection{Retained row receives sufficiently many fresh replay samples.}
\label{app:subsubsec:replay_samples}
We first bound the learning counts and weighted occupancies, using the histories $\cH_{k-1}^{\ell}$, $\cA_k^{\ell}$, and $\cB_k^{\ell}$ from Appendix~\ref{app:subsubsec:proof_of_optimistic_average_counterfactual}.
\begin{lemma}
\label{lemma:pessimistic_occupancy_bounds}
For each fixed phase $\ell$ and tuple $(h,s,a)$, conditional on $\cH_0^{\ell}$,
\begin{align}
D_{\ell,h,s,a}^{\mathrm{avg}}
&\leq2N_{\ell,h}(s,a)+2\log\frac1{\delta''},\label{equ:pessimistic_occ_1}\\
N_{\ell,h}(s,a)
&\leq2D_{\ell,h,s,a}^{\mu}+2\log\frac1{\delta''},\label{equ:pessimistic_occ_3}
\end{align}
with failure probabilities at most $(\delta'')^{3/2}$ and $(\delta'')^2$, respectively.
\end{lemma}
\begin{proof}
Let 
\[
X_{\ell,h,s,a}^k:=\mathbbm{1}_{\{(s_h^{\ell,k},a_h^{\ell,k})=(s,a)\}}
\]
so the sample count can be expressed as $N_{\ell,h}(s,a)=\sum_{k=1}^K X_{\ell,h,s,a}^k$. Recall the completed-history filtration $(\cH_k^{\ell})$ and the two intermediate sigma-fields $\cA_k^{\ell}=\cH_{k-1}^{\ell}\vee\sigma(s_1^{\ell,k})$ and $\cB_k^{\ell}=\cH_{k-1}^{\ell}\vee\sigma(m^{\ell,k})$ defined in~\cref{equ:sigma_fields}. Conditional on $\cH_{k-1}^{\ell}$, all counterfactual policies are fixed and $m^{\ell,k}$ and $s_1^{\ell,k}$ are independent. The conditional independence of the reward and initial-state draws gives
\[
\E[X_{\ell,h,s,a}^k\mid\cA_k^{\ell}]
=\sum_{m=1}^M p_\ell(m)d_{\ell,h}^{k,m}(s,a),\qquad
\E[X_{\ell,h,s,a}^k\mid\cB_k^{\ell}]
=d_h^{\pi^{\ell,k},\mu}(s,a).
\]
\paragraph{Proof of~\eqref{equ:pessimistic_occ_1}.} Denote the weighted occupancy as
\[
\overline{d}_{\ell,h}^{k}(s,a):=\sum_{m=1}^M p_\ell(m)d_{\ell,h}^{k,m}(s,a)
\]
Random variable $X_{\ell, h, s,a}^{k}$ satisfies
\[
\mathbb{E}\sqbk{X_{\ell,h,s,a}^{k}|\cA_k^{\ell}} = \sum_{m=1}^M p_\ell(m)d_{\ell,h}^{k,m}(s,a)= \overline{d}_{\ell,h}^{k}(s,a)
\]
For a Bernoulli variable $X$ with a conditional mean $\E[X|\cF] = u$ for some sigma-field $\cF$
\begin{equation}
    \label{equ:bernoulli_martingale_bounded_by_1}
    \mathbb{E}\sqbk{e^{3u/4 - (\log 4) X} | \cF} = e^{3u/4}(1-u) + \frac{ue^{3u/4}}{4}= e^{3u/4}\bracket{1 - \frac{3u}{4}} \leq 1
\end{equation}
Consider the random process
\[
Y_{k} := \exp\bracket{\frac{3}{4}\sum_{k'=1}^{k} \overline{d}_{\ell,h}^{k'}(s,a) - (\log 4)\sum_{k'=1}^{k} X_{\ell,h,s,a}^{k'}}
\]
Then by ~\eqref{equ:bernoulli_martingale_bounded_by_1}
\[
\mathbb{E}\sqbk{Y_k \mid \cA_k^{\ell}}
=Y_{k-1}\mathbb{E}\sqbk{
e^{3\overline{d}_{\ell,h}^{k}(s,a)/4-(\log4)X_{\ell,h,s,a}^{k}}
\mid\cA_k^{\ell}}\leq Y_{k-1}.
\]
Since $\cH_{k-1}^{\ell}\subseteq\cA_k^{\ell}\subseteq\cH_k^{\ell}$, the tower property
shows that $(Y_k)$ is a nonnegative supermartingale with respect to
$(\cH_k^{\ell})$, with $Y_0=1$. By Ville's inequality, conditional on $\cH_0^{\ell}$, with probability
at least $1-(\delta'')^{3/2}$,
\[
Y_K\leq(\delta'')^{-3/2}\implies
\frac34D_{\ell,h,s,a}^{\mathrm{avg}}-(\log4)N_{\ell,h}(s,a)
\leq\frac32\log\frac1{\delta''}.
\]
Since $\frac43\log4\leq2$, this implies
\eqref{equ:pessimistic_occ_1} with failure probability at most
$(\delta'')^{3/2}$, i.e.
\[
D_{\ell, h, s, a}^{\mathrm{avg}}\leq 2N_{\ell, h}(s,a) + 2\log\frac{1}{\delta''}
\]
with probability at least $1-(\delta'')^{3/2}$.

\paragraph{Proof of equation~\ref{equ:pessimistic_occ_3}} Conditional on $\cB_k^{\ell}$, the roll-out policy $\pi^{\ell,k}=\pi^{\ell,k,m^{\ell,k}}$ is fixed, while the initial state $s_1^{\ell,k}\sim\mu$ has not yet been observed. Hence,
\[
\mathbb{E}\!\left[X_{\ell,h,s,a}^k\mid\cB_k^{\ell}\right]
=d_h^{\pi^{\ell,k},\mu}(s,a).
\]
For a Bernoulli variable $X$ with conditional mean $\E\sqbk{X|\cF} = u$ for some sigma-field $\cF$
\begin{equation}
\label{equ:bernoulli_martingale_bounded_by_1_2}
    \mathbb{E}\sqbk{e^{X-(e-1)u}|\cF}
    =e^{-(e-1)u}\bracket{eu+1-u}
    =e^{-(e-1)u}\bracket{1+(e-1)u}\leq1
\end{equation}
Define the random process
\[
Z_{k}
:=\exp\left(
\sum_{k'=1}^k X_{\ell,h,s,a}^{k'}
-(e-1)\sum_{k'=1}^k d_h^{\pi^{\ell,k'},\mu}(s,a)
\right),
\qquad Z_0=1.
\]
Since $Z_{k-1}$ is $\cH_{k-1}^{\ell}$-measurable and hence $\cB_{k}^{\ell}$-measurable,
\[
\mathbb{E}\!\left[Z_{k}\mid\cB_k^{\ell}\right]
=
Z_{k-1}
\mathbb{E}\!\left[
\exp\left(
X_{\ell,h,s,a}^k
-(e-1)d_h^{\pi^{\ell,k},\mu}(s,a)
\right)
\middle|\cB_k^{\ell}
\right]
\leq Z_{k-1}.
\]
Taking conditional expectations again given $\cH_{k-1}^{\ell}$ shows that
$(Z_k)$ is a nonnegative supermartingale with respect to $(\cH_k^{\ell})$.

By Ville's inequality, conditional on $\cH_0^{\ell}$, with probability at least $1-(\delta'')^2$,
\[
\sum_{k=1}^K X_{\ell,h,s,a}^k
-(e-1)\sum_{k=1}^K d_h^{\pi^{\ell,k},\mu}(s,a)
\leq 2\log\frac{1}{\delta^{\prime\prime}}.
\]
Using
$N_{\ell,h}(s,a)=\sum_{k=1}^K X_{\ell,h,s,a}^k$, the definition of $D_{\ell,h,s,a}^{\mu}$, and $e-1<2$, we obtain
\[
N_{\ell,h}(s,a)
\leq 2D_{\ell,h,s,a}^{\mu}+2\log\frac{1}{\delta^{\prime\prime}}.
\]
This proves~\eqref{equ:pessimistic_occ_3} with conditional failure
probability at most $(\delta'')^2$. Integrating over the phase-start
history and taking a union bound over all phases and rows gives
failure probabilities at most $SAHL(\delta'')^{3/2}$ and
$SAHL(\delta'')^2$ for~\eqref{equ:pessimistic_occ_1} and
\eqref{equ:pessimistic_occ_3}, respectively. In particular, the first
event alone has failure probability at most $SAHL\delta''$, as used
in Appendix~\ref{app:subsec:proof_of_pessimistic_average_counterfactual}.
Their joint failure probability is at most
$SAHL\bracket{(\delta'')^{3/2}+(\delta'')^2}\leq2SAHL\delta''$.
\end{proof}

\begin{lemma}[Formal statement of Lemma~\ref{lemma:replay_count_bound}]
With probability at least $1-2SAHL\delta''$, the number of the replay samples collected in the sealed replay stage satisfies
\[
\widetilde N_{\ell,h}(s,a)=N_{\ell,h}(s,a),\qquad
\ell\in[L],\quad(h,s,a)\in\cR_\ell.
\]
\end{lemma}
\begin{proof}
On the learning-count event~\eqref{equ:pessimistic_occ_3}, every retained row satisfies
\begin{equation}
\label{equ:occupancy_lowerbound}
D_{\ell,h,s,a}^{\mu}
\geq\frac{N_{\ell,h}(s,a)-2\log(1/\delta'')}{2}
\geq\frac38N_{\ell,h}(s,a),
\end{equation}
since $N_{\ell,h}(s,a)\geq8\log(1/\delta'')$. Conditional on the history $\cF_\ell$ immediately before replay, all roll-out policies, retained rows, and quotas are fixed. Let $X_{\ell,h,s,a}^{k,l}$ indicate a visit to $(h,s,a)$ in the $l$-th fresh replay of the indexed policy $\pi^{\ell,k}$, and let
\[
T_{\ell,h,s,a}:=\sum_{k=1}^K\sum_{l=1}^{L_{\mathrm{rep}}}X_{\ell,h,s,a}^{k,l}.
\]
The replay episodes are conditionally independent, including episodes that replay identical policies. For $L_{\mathrm{rep}}=8$, on the learning-count event,
\[
\E[T_{\ell,h,s,a}\mid\cF_\ell]
=8D_{\ell,h,s,a}^{\mu}\geq3N_{\ell,h}(s,a).
\]
Since the sealed replay phase only records the first $N_{\ell, h}(s,a)$ samples, 
\[
\widetilde N_{\ell,h}(s,a)=\min\{T_{\ell,h,s,a},N_{\ell,h}(s,a)\}.
\]
Then, the Chernoff's inequality gives,
\begin{align*}
\mathbb P\bracket{\widetilde N_{\ell,h}(s,a)<N_{\ell,h}(s,a)\mid\cF_\ell}
\leq&\mathbb P\bracket{T_{\ell,h,s,a}\leq\tfrac13\E[T_{\ell,h,s,a}\mid\cF_\ell]\mid\cF_\ell}\\
\leq&\exp\bracket{-\tfrac23N_{\ell,h}(s,a)}\\
\leq&(\delta'')^{16/3}.
\end{align*}
Using the failure bound $(\delta'')^2$ for~\eqref{equ:pessimistic_occ_3}, a union bound over all phases and tuples bounds the probability of any unfilled $\widetilde{N}_{\ell, h}(s,a)\ <N_{\ell, h}(s,a)$ by
\[
SAHL\bracket{(\delta'')^2+(\delta'')^{16/3}}\leq2SAHL\delta''.
\]
This demonstrates that, with probability at least $1-2SAHL\delta''$, across all $L$ phases, samples drawn via sealed replay always include as many new samples as were collected during the learning phase. Then we complete the proof.
\end{proof}

\subsubsection{Proof of Lemma~\ref{lemma:pessimism}}
\label{app:subsub:proof_of_pessimism}
\begin{lemma}[Formal statement of Lemma~\ref{lemma:pessimism}]
\label{lemma:app:pessimism}
For every $\ell\in[L]$, it holds that $
    \mathbb{E}\sqbk{\exp\bracket{\frac{\hat{v}_{\ell, m}^{-} - v_{\ell,m}}{H}}|\cF_{\ell}} \leq 1$.
Moreover, with probability at least $1-\delta/64$, it holds that $\sum_{\ell=1}^L(\hat v_{\ell,m}^--v_{\ell,m})
\le H\log\frac{64M}{\delta}$
 simultaneously for all $m\in[M]$,
\end{lemma}
\begin{proof}
We will prove that $ \mathbb{E}\sqbk{\exp\bracket{\frac{\hat{v}_{\ell, m}^{-} - v_{\ell,m}}{H}}|\cF_{\ell}}\leq 1$.
The final tail probability bound follows by Markov's inequality and the bound $\mathbb{E}\sqbk{\exp\bracket{\sum_{\ell=1}^L\frac{\hat{v}_{\ell, m}^{-} - v_{\ell,m}}{H}}}\leq 1$.

Conditional on $\cF_\ell$, the counterfactual policies $\{(\pi^{\ell, k, m})_{k\in [K]}\}$, initial states $(s_{1}^{\ell, k})_{k\in [K]}$, retained set $\cR_\ell$, and phase learning tuple count $N_{\ell,h}(s,a)$ are fixed. Represent the fresh replay transitions by arrays $
Z_{\ell,h,s,a}^{(i)}\overset{\mathrm{i.i.d.}}{\sim}P_{h,s,a}$ for $i\geq 1$.
 The auxiliary run uses the fixed-prefix model
\[
\widetilde P_{\ell,h,s,a}
=\frac1{N_{\ell,h}(s,a)}\sum_{i=1}^{N_{\ell,h}(s,a)}\delta_{Z_{\ell,h,s,a}^{(i)}},
\qquad(h,s,a)\in\cR_\ell,
\]
Under the success event of Lemma~\ref{lemma:replay_count_bound}, $\widetilde{P}_{\ell, h, s,a}$ is constructed via $N_{\ell, h}(s,a)$ fresh replayed samples. The recursion gives $0\leq\underline V_h\leq H-h+1$, and its clipping at zero is inactive because rewards and continuation values are nonnegative. For a retained row, put $n=N_{\ell,h}(s,a)$, so
\begin{align}
\underline Q_h^{\ell, k, m}(s,a)=r_h^m(s,a)+\frac1{n+1}\sum_{i=1}^n\underline V_{h+1}^{\ell, k, m}(Z_{\ell,h,s,a}^{(i)}).\label{eq:defq}
\end{align}
We first record the scalar inequality for i.i.d. variables $Y_1,\ldots,Y_n\in[0,H]$:
\begin{align}
\E\sqbk{\exp\bracket{\frac{\sum_{i=1}^nY_i}{H(n+1)}}} \leq\left[1+\frac{\E\sqbk{Y_1}}{H}\bracket{e^{1/(n+1)}-1}\right]^n\leq\exp\bracket{\frac{\E\sqbk{Y_1} }{H}}.\label{eq:scalar}
\end{align}
The first inequality follows from convexity of the exponential on $[0,1]$. The second uses $1+x\leq e^x$ and $n(e^{1/(n+1)}-1)\leq1$ holds for all $n\geq 0$; the latter follows from $e^x\leq(1-x)^{-1}$ for $0\leq x<1$.

We prove by backward induction that, for every $h$ and $s$,
\begin{equation}
\label{equ:exponential_value_bound}
\E\sqbk{\exp\bracket{\frac{\underline V_h^{\ell, k, m}(s)-V_h^{\pi,m}(s)}{H}}\mid\cF_\ell}\leq1.
\end{equation}
\textbf{Base case.} The terminal values are zero since $\underline{V}_{H+1}^{\ell, k, m}(s)=V_{H+1}^{\pi^{\ell, k, m},m}=0$, the base case holds.

\textbf{Inductive Hypothesis.} Suppose
\begin{equation}
\label{equ:inductive_hypothesis}
    \E\sqbk{\exp\bracket{\frac{\underline V_{h+1}^{\ell, k, m}(s)-V_{h+1}^{\pi,m}(s)}{H}}\mid\cF_\ell}\leq1.
\end{equation}
\textbf{Inductive Step.} Conditional on $\cF_\ell$ and $\underline V_{h+1}^{\ell, k, m}$, the current tuple samples remain i.i.d., since $\underline V_{h+1}^{\ell, k, m}$ depends only on replay samples at stages strictly greater than $h$.

On the tuple $(h,s,a)\in \cR_{\ell}$, recall the definition of $\underline{Q}_{h}^{\ell, k, m}(s,a)$ in~\eqref{eq:defq}.
\begin{align*}
    \underline Q_h^{\ell, k, m}(s,a)=&r_h^m(s,a)+\frac1{n+1}\sum_{i=1}^n\underline V_{h+1}^{\ell, k, m}(Z_{\ell,h,s,a}^{(i)})\\
    Q_{h}^{\pi^{\ell, k, m}, m}(s,a)=& r_{h}^{m}(s,a) + \langle P_{h,s,a}, V_{h+1}^{\pi^{\ell, k, m}, m}\rangle
\end{align*}
Subtract the equations we have
\[
\underline Q_h^{\ell, k, m}(s,a) - Q_{h}^{\pi^{\ell, k, m}, m}(s,a) =\frac1{n+1}\sum_{i=1}^n\underline V_{h+1}^{\ell, k, m}(Z_{\ell,h,s,a}^{(i)}) - \langle P_{h,s,a}, V_{h+1}^{\pi^{\ell, k, m}, m}\rangle
\]
Conditional on $\cF_{\ell}$, $\underline{V}_{h+1}^{\ell, k, m}$
\[
\E\sqbk{\underline{V}_{h+1}^{\ell, k, m}(Z_{\ell, h, s, a}^{(i)})|\cF_{\ell}, \underline{V}_{h+1}^{\ell, k, m}} = \langle P_{h,s,a}, \underline{V}_{h+1}^{\ell, k, m} \rangle
\]
Then apply~\cref{eq:scalar}, with $Y_i=\underline V_{h+1}^{\ell, k, m}(Z_{\ell,h,s,a}^{(i)})$
\[
\E\sqbk{\exp\bracket{\frac{\sum_{i=1}^{n}\underline{V}_{h+1}^{\ell, k, m}(Z_{\ell, h,s,a}^{(i)})}{H(n+1)}}|\cF_{\ell}, \underline{V}_{h+1}^{\ell, k, m}} \leq \exp\bracket{\frac{\langle P_{h,s,a}, \underline{V}_{h+1}^{\ell, k, m} \rangle}{H}}
\]
We derive
\begin{align*}
    &\E\sqbk{\exp\bracket{\frac{\underline{Q}_{h}^{\ell, k, m}(s,a) - Q_{h}^{\pi^{\ell, k, m}, m}(s,a)}{H}}|\cF_{\ell}, \underline{V}_{h+1}^{\ell, k, m}}\\
    =&\E\sqbk{\exp\bracket{\frac{\sum_{i=1}^{n}\underline{V}_{h+1}^{\ell, k, m}(Z_{\ell, h,s,a}^{(i)})}{H(n+1)}}|\cF_{\ell}, \underline{V}_{h+1}^{\ell, k, m}} \times \exp\bracket{-\frac{\langle P_{h,s,a}, V_{h+1}^{\pi^{\ell, k, m}, m} \rangle}{H}}\\
    \leq&\exp\bracket{\frac{\langle P_{h,s,a}, \underline{V}_{h+1}^{\ell, k, m} \rangle}{H}} \times \exp\bracket{-\frac{\langle P_{h,s,a}, V_{h+1}^{\pi^{\ell, k, m}, m} \rangle}{H}}\\
    \leq& \exp\bracket{\frac{\langle P_{h,s,a}, \underline{V}_{h+1}^{\ell, k, m} - V_{h+1}^{\pi^{\ell, k, m}, m} \rangle}{H}}
\end{align*}
By the tower rule, since $\cF_{\ell}\subseteq \sigma (\cF_{\ell}, \underline{V}_{h+1}^{\ell, k, m})$
\begin{align*}
    &\E\sqbk{\exp\bracket{\frac{\underline{Q}_{h}^{\ell, k, m}(s,a) - Q_{h}^{\pi^{\ell, k, m}, m}(s,a)}{H}}|\cF_{\ell}}\\
    =&\E\sqbk{\E\sqbk{\exp\bracket{\frac{\underline{Q}_{h}^{\ell, k, m}(s,a) - Q_{h}^{\pi^{\ell, k, m}, m}(s,a)}{H}}|\cF_{\ell}, \underline{V}_{h+1}^{\ell, k, m}}\mid \cF_{\ell}}\\
    \leq& \E\sqbk{\exp\bracket{\frac{\langle P_{h,s,a}, \underline{V}_{h+1}^{\ell, k, m} - V_{h+1}^{\pi^{\ell, k, m}, m} \rangle}{H}}\mid\cF_{\ell}}\\
    =&
\mathbb{E}\left[
  \prod_{s'\in\mathcal{S}}\exp\left(
    \frac{ P_{h,s,a}(s')(\underline{V}_{h+1}^{\ell,k,m}(s')
    - V_{h+1}^{\pi^{\ell,k,m},m}(s'))}{H}
  \right)
  \,\middle|\, \mathcal{F}_\ell
\right] \\
\leq&
\prod_{s'\in\mathcal{S}}\mathbb{E}\left[
  \exp\left(
    \frac{\underline{V}_{h+1}^{\ell,k,m}(s')
    - V_{h+1}^{\pi^{\ell,k,m},m}(s')}{H}
  \right)
  \,\middle|\, \mathcal{F}_\ell
\right]^{P_{h,s,a}(s')}\\
\leq& 1.
\end{align*}
The second inequality follows from H\"older's inequality that $\mathbb{E}[\prod_{i}X_i^{a_i}] \leq \prod_i(\mathbb{E}[X_i])^{a_i}$  for non-negative $X_i$ and $a_i\in (0,1)$ for each $i$ satisfying that $\sum_i a_i = 1$. We derive the last inequality from~\cref {equ:inductive_hypothesis}.

On a tuple $(h,s,a)\notin \cR_{\ell}$,
\[
\underline Q_h^{\ell, k, m}(s,a)=r_h^m(s,a)\leq Q_h^{\pi^{\ell, k, m},m}(s,a)
\]
so the same bound holds directly. Finally, Jensen's inequality for the action average yields
\begin{align*}
    &\E\sqbk{\exp\bracket{\frac{\underline{V}_{h}^{\ell, k, m}(s) - V_{h}^{\pi^{\ell, k, m}, m}(s)}{H}}\mid\cF_\ell}\\
\leq&\sum_{a\in\cA}\pi_h^{\ell, k, m}(a| s)\E\sqbk{\exp\bracket{\frac{\underline{Q}_{h}^{\ell, k, m}(s, a) - Q_{h}^{\pi^{\ell, k, m}, m}(s, a)}{H}}\mid\cF_\ell}\\
\leq&1,
\end{align*}
completing the induction.

Since $e^x$ is a convex function, Jensen's inequality gives, for all $\ell\geq 1$
\begin{equation}
    \begin{aligned}
        &\E\sqbk{\exp\bracket{\frac{\hat{v}_{\ell, m}^{-} - v_{\ell, m}}{H}}\mid \cF_{\ell}}\\
        \leq& \frac{1}{K}\sum_{k=1}^{K}\E\sqbk{\exp\bracket{\frac{\underline{V}_{1}^{\ell, k, m}(s_{1}^{\ell, k}) - V_{1}^{\pi^{\ell, k, m}, m}(s_{1}^{\ell, k})}{H}}\mid \cF_{\ell}}\leq 1\label{equ:exponential_gap_bound}
    \end{aligned}
\end{equation}
Since $\cF_\ell$ contains all earlier optimistic evaluations and conservative evaluations, iterating conditional expectations using tower rule across phases gives
\begin{align*}
    &\E\sqbk{\exp\bracket{\frac{1}{H}\sum_{\ell=1}^{L}(\hat{v}_{\ell, m}^{-} - v_{\ell, m})}}\\
    =& \E\sqbk{\exp\bracket{\frac{1}{H}\sum_{\ell=1}^{L-1}(\hat{v}_{\ell, m}^{-} - v_{\ell, m})}\E\sqbk{\exp\bracket{\frac{\hat{v}_{L, m}^{-} - v_{L, m}}{H}}\mid \cF_{L}}}\\
    \overset{(i)}{\leq}& \E\sqbk{\exp\bracket{\frac{1}{H}\sum_{\ell=1}^{L-1}(\hat{v}_{\ell, m}^{-} - v_{\ell, m})}}\\
    \leq& \cdots \leq 1
\end{align*}
Where $(i)$ is derived by the~\cref{equ:exponential_gap_bound}.

\end{proof}

\subsection{Proof of Lemma~\ref{lemma:pessimistic_average_counterfactual}}
\label{app:subsec:proof_of_pessimistic_average_counterfactual}
We now bound the conditional mean evaluation error using the same auxiliary fixed-prefix model.
\begin{lemma}[Formal statement of Lemma~\ref{lemma:pessimistic_average_counterfactual}]
\label{lemma:app:pessimistic_average_counterfactual}
With probability at least $1-SAHL\delta''$, for every $\ell\in[L]$,
\[
0\leq\sum_{m=1}^M p_\ell(m)\bracket{v_{\ell,m}-\E[\hat v_{\ell,m}^-\mid\cF_\ell]}
\leq\frac{18SAH^2}{K}\log\frac1{\delta''}.
\]
\end{lemma}
\begin{proof}
For a fixed phase $\ell$ and a counterfactual policy $\pi=\pi^{\ell,k,m}$, we suppress $(\ell,k,m)$ on the conservative evaluation, i.e., denote $\underline{Q}_{h}(s,a) = \underline{Q}_{h}^{\ell, k, m}(s,a)$, $\underline{V}_{h}(s) = \underline{V}_{h}^{\ell, k, m}(s)$, $r=r^{m}$ in this proof. Conditional on $\cF_{\ell}$
\[
\E[\underline Q_h(s,a)\mid\cF_\ell]
=\begin{cases}
r_h^m(s,a)+\dfrac{N_{\ell,h}(s,a)}{N_{\ell,h}(s,a)+1}
\langle P_{h,s,a},\E[\underline V_{h+1}\mid\cF_\ell]\rangle,
&(h,s,a)\in\cR_\ell,\\
r_h^m(s,a),&(h,s,a)\notin\cR_\ell.
\end{cases}
\]
We first show 
\[
\E\sqbk{\underline{V}_{h}(s)|\cF_{\ell}} \leq  V_{h}^{\pi, m}(s)
\]
by induction.

\textbf{Base case.} The base case holds trivially since $\underline{V}_{H+1}(s)= V_{H+1}^{\pi, m}(s)=0$.

\textbf{Inductive hypothesis.} Suppose $\E[\underline{V}_{h+1}(s)|\cF_{\ell}]\leq V_{h+1}^{\pi, m}(s)$ for all $s\in \cS$

\textbf{Inductive step.} On retained set $\cR$, backward induction implies
\begin{align*}
    \E\sqbk{\underline{Q}_{h}(s,a)|\cF_{\ell}} =& r_{h}^{m}(s,a) + \dfrac{N_{\ell,h}(s,a)}{N_{\ell,h}(s,a)+1}
\langle P_{h,s,a},\E[\underline V_{h+1}\mid\cF_\ell]\rangle\\
\leq& r_{h}^{m}(s,a) + \langle P_{h,s,a},\E[\underline V_{h+1}\mid\cF_\ell]\rangle\\
\leq& r_{h}^{m}(s,a) + \langle P_{h,s,a},V_{h+1}^{\pi, m}\rangle\\
=& Q_{h}^{\pi, m}(s,a)
\end{align*}
For an unretained tuple $(h,s,a)\notin \cR_{\ell}$, 
\[
\E\sqbk{\underline{Q}_{h}(s,a)|\cF_{\ell}} = r_{h}^{m}(s,a)\leq Q_{h}^{\pi, m}(s,a)
\]
Together with $\E[\underline V_h(s)\mid\cF_\ell]=\sum_a\pi_h(a| s)\E[\underline Q_h(s,a)\mid\cF_\ell]$, we prove
\[
0\leq\E[\underline V_h(s)\mid\cF_\ell]\leq V_h^{\pi,m}(s)\leq H.
\]
For a retained tuple $(h,s,a)\in\cR_{\ell}$, subtracting the mean recursion from the true Bellman equation yields
\begin{align*}
&Q_h^{\pi,m}(s,a)-\E[\underline Q_h(s,a)\mid\cF_\ell]\\
&=\langle P_{h,s,a},V_{h+1}^{\pi,m}-\E[\underline V_{h+1}\mid\cF_\ell]\rangle
+\frac{\langle P_{h,s,a},\E[\underline V_{h+1}\mid\cF_\ell]\rangle}{N_{\ell,h}(s,a)+1}\\
&\leq\langle P_{h,s,a},V_{h+1}^{\pi,m}-\E[\underline V_{h+1}\mid\cF_\ell]\rangle
+\frac H{N_{\ell,h}(s,a)+1}.
\end{align*}
For an unretained tuple $(h,s,a)\notin \cR_{\ell}$ the same upper bound holds with the last term replaced by $H$, because only the immediate reward is retained:
\[
Q_{h}^{\pi, m}(s,a) - \E\sqbk{\underline{Q}_{h}(s,a)|\cF_{\ell}} \leq\langle P_{h,s,a},V_{h+1}^{\pi,m}-\E[\underline V_{h+1}\mid\cF_\ell]\rangle + H
\]

Averaging over actions and unrolling along the true transition kernel and policy $\pi$, starting from the recorded initial state $s_1^{\ell,k}$, recall the definition of $d_{\ell, h}^{k, m}(s,a)$ in~\cref{equ:occupancy} and plug back the index $(\ell, k, m)$
\begin{align*}
0\leq& V_1^{\pi^{\ell, k, m},m}(s_1^{\ell,k})
-\E[\underline V_1^{\ell, k, m}(s_1^{\ell,k})\mid\cF_\ell]\\
\leq& H\left[
\sum_{(h,s,a)\in\cR_\ell}\frac{d_{\ell,h}^{k,m}(s,a)}{N_{\ell,h}(s,a)+1}
+\sum_{(h,s,a)\notin\cR_\ell}d_{\ell,h}^{k,m}(s,a)\right].
\end{align*}
Taking the average over $k\in[K]$ and then weighting by the reward-sampling distribution $p_\ell$ gives
\begin{equation}
\label{equ:mean_evaluation_bias}
0\leq\sum_{m=1}^M p_\ell(m)\bracket{v_{\ell,m}-\E[\hat v_{\ell,m}^-\mid\cF_\ell]}\quad\leq\frac HK\left[\underbrace{\sum_{(h,s,a)\in\cR_\ell}\frac{D_{\ell,h,s,a}^{\mathrm{avg}}}{N_{\ell,h}(s,a)+1}}_{\text{retained part}}
+\underbrace{\sum_{(h,s,a)\notin\cR_\ell}D_{\ell,h,s,a}^{\mathrm{avg}}}_{\text{unretained part}}\right].
\end{equation}
The right-hand side can be decomposed into the retained part and the unretained part. 

\textbf{Retained part.} By Lemma~\ref{lemma:pessimistic_occupancy_bounds}, on the weighted occupancy event~\eqref{equ:pessimistic_occ_1}, a retained row satisfies
\begin{equation}
\label{equ:retained_occupancy_average_bound}
\frac{D_{\ell,h,s,a}^{\mathrm{avg}}}{N_{\ell,h}(s,a)+1}
\leq\frac{2N_{\ell,h}(s,a)+2\log(1/\delta'')}{N_{\ell,h}(s,a)+1}
\leq\frac94
\end{equation}
since $\log(1/\delta'')\geq1$, every tuple contributes at most $18\log(1/\delta'')$ to the retained part. Since there are $SAH$ tuples, then
\[
\text{retained part}\leq \frac{9SAH}{4}
\]

\textbf{Unretained part.} For an unretained triple $(h,s,a)\notin \cR_{\ell}$, we have $N_{\ell,h}(s,a)<8\log(1/\delta'')$. Hence, by~\cref{equ:pessimistic_occ_1}
\[
D_{\ell,h,s,a}^{\mathrm{avg}}
\leq2N_{\ell,h}(s,a)+2\log\frac1{\delta''}
<18\log\frac1{\delta''}.
\]
Similarly, $|\cR_{\ell}^{c}|\leq SAH$ implies
\[
\text{unretained part}\leq 18SAH\log\frac{1}{\delta''}.
\]
Combining the two parts together, we have
\[
0\leq \sum_{m=1}^{M}p_{\ell}(m)\bracket{v_{\ell, m} - \E\sqbk{\hat{v}_{\ell, m}^{-}|\cF_{\ell}}} \leq \frac{18SAH^2}{K}\log\frac{1}{\delta''}.
\]
 Lemma~\ref{lemma:pessimistic_occupancy_bounds} ensures this event simultaneously over all phases and rows with failure probability at most $SAHL(\delta'')^{3/2}\leq SAHL\delta''$, proving the claim. The same event controls all rewards through their weighted occupancy; no additional union bound over rewards is required.
\end{proof}

\subsection{Proof of Theorem~\ref{thm:adv_sample_complexity}}
\label{app:proof_of_adv_sample_complexity}
In this section, we complete the proof of our main Theorem~\ref{thm:adv_sample_complexity}. 

\paragraph{Bounding the cumulative certificates.}
Fix $\ell\in[L]$ and $m\in[M]$. Conditional on $\cF_\ell$, both
$\hat v_{\ell,m}^+$ and $v_{\ell,m}$ are fixed. Denote the $X$ as the normalized gap between optimistic evaluation and the conservative evaluation: $
X=\frac{\max\{\hat{v}_{\ell, m}^{+}, v_{\ell,m}\}-\hat v_{\ell,m}^-}{H}
\in[-1,1]$. Then
Lemma~\ref{lemma:app:pessimism} implies
$\E[e^{-X}\mid\cF_\ell]\leq1$. For $x\in[-1,1]$, the elementary
inequality $e^{-x}\geq1-x+x^2/3$ therefore gives $
\mathbb{E}[X^2 |\mathcal{F}_{\ell}] \leq 3 \mathbb{E}[X|\mathcal{F}_{\ell}]$. Also noting that  $e^x\leq1+x+x^2$, we then have  $
\E[\exp(X)\mid\cF_\ell]\leq1+4\E[X\mid\cF_\ell].$

Since $(\hat v_{\ell,m}^+-\hat v_{\ell,m}^-)/H\leq X$, it follows that
\begin{align}
\E\!\left[\exp\!\left(\frac{\hat v_{\ell,m}^+-\hat v_{\ell,m}^-}{H}\right)
\middle|\cF_\ell\right]\leq & 1+\frac4H\left[(\hat v_{\ell,m}^+-v_{\ell,m})_++v_{\ell,m}-\E[\hat v_{\ell,m}^-\mid\cF_\ell]\right],
\label{eq:certificate-exponential-moment}
\end{align}
where $(x)_+ = \max\{x,0\}$.
We recall that $W_{\ell} =
    \frac{1}{M}\sum_{m=1}^M
    \exp\left(
        \frac{1}{H}\sum_{j=1}^{\ell-1}
        (\hat v_{j,m}^+-\hat v_{j,m}^-)
    \right)$. In particular $W_1 = 1$.
    
The update in Algorithm~\ref{alg:main} gives $
\frac{W_{\ell+1}}{W_\ell}
=\sum_{m=1}^M p_\ell(m)
\exp\!\left(\frac{\hat v_{\ell,m}^+-\hat v_{\ell,m}^-}{H}\right).$

Then, by~\cref{eq:certificate-exponential-moment} it implies
\begin{align*}
\E\sqbk{\frac{W_{\ell+1}}{W_{\ell}}\mid \cF_{\ell}} =&\E\!\left[\sum_{m\in [M]}p_{\ell}(m)\exp\!\left(\frac{\hat v_{\ell,m}^+-\hat v_{\ell,m}^-}{H}\right)
\middle|\cF_\ell\right]
\\
\leq&1+\sum_{m\in [M]}p_{\ell}(m)\bracket{\frac4H\left[(\hat v_{\ell,m}^+-v_{\ell,m})_+
+v_{\ell,m}-\E[\hat v_{\ell,m}^-\mid\cF_\ell]\right]}\\
\leq & \exp\bracket{\sum_{m\in [M]}p_{\ell}(m)\bracket{\frac4H\left[(\hat v_{\ell,m}^+-v_{\ell,m})_+
+v_{\ell,m}-\E[\hat v_{\ell,m}^-\mid\cF_\ell]\right]}}
\end{align*}
As a result, divide both sides by the right-hand side.
\begin{equation}
\label{equ:martingale_exponential_smaller_than_1}
    \E\sqbk{\frac{W_{\ell+1}}{W_{\ell}}\exp\!\left\{-\frac4H
\sum_{m=1}^M p_\ell(m)
\left[(\hat v_{\ell,m}^+-v_{\ell,m})_+
+v_{\ell,m}-\E[\hat v_{\ell,m}^-\mid\cF_\ell]\right]
\right\}\mid \cF_{\ell}}\leq 1
\end{equation}
Iterating conditional expectations over phases yields:
\begin{align*}
    &\E\!\left[W_{L+1}\exp\!\left\{-\frac4H
    \sum_{\ell=1}^L\sum_{m=1}^M p_\ell(m)
    \left[(\hat v_{\ell,m}^+-v_{\ell,m})_+
    +v_{\ell,m}-\E[\hat v_{\ell,m}^-\mid\cF_\ell]\right]
    \right\}\right] \leq 1.
\end{align*}
Since $
W_{L+1}\geq \frac{1}{M}\exp\bracket{\frac{1}{H}\sum_{\ell=1}^{L}
(\hat v_{\ell,m}^+-\hat v_{\ell,m}^-)}$ 
for every $m$,
Markov's inequality gives, with probability at least $1-\delta/64$,
\begin{equation}
    \label{eq:cumulative-certificate-bound}
    \begin{aligned}
        \max_{m\in[M]}\sum_{\ell=1}^L(\hat v_{\ell,m}^+-\hat v_{\ell,m}^-)
\leq&4\sum_{\ell=1}^L\sum_{m=1}^M p_\ell(m)
\left[(\hat v_{\ell,m}^+-v_{\ell,m})_{+}
+v_{\ell,m}-\E[\hat v_{\ell,m}^-\mid\cF_\ell]\right]\\
&\qquad +H\log\frac{64M}{\delta}.
    \end{aligned}
\end{equation}
On the optimism event, the positive part can be removed. Combining
Lemma~\ref{lemma:optimistic_average_counterfactual} with
Lemma~\ref{lemma:app:pessimistic_average_counterfactual} then gives
\begin{align}
\frac1L\max_{m\in[M]}\sum_{\ell=1}^L
(\hat v_{\ell,m}^+-\hat v_{\ell,m}^-)
&\lesssim\sqrt{\frac{SAH^3}{KL}\log_2^3(2KL)\log\frac M{\delta'}}
+\frac{SAH^2}{K}\log\frac1{\delta''}
\nonumber\\
&\quad+\frac HL\log\frac{64M}{\delta}
+\frac H{KL}\log\frac{64}{\delta}.
\label{eq:average-cumulative-certificate}
\end{align}
We account for these events together below.

\paragraph{Correctness of certification.}
Recall the definition of $ \mathcal{F}_{k-1}^{\ell}$. By definition of 
$s_1^{\ell,k}\sim\mu$, we have that 
\[
\E\sqbk{\frac{V_{1}^{*, m}(s_{1}^{\ell, k}) - V_1^{\pi^{\ell,k,m},m}(s_{1}^{\ell, k})}{H}\mid \cF_{k-1}^{\ell}} = \frac{V_0^{*,m}-V_0^{\pi^{\ell,k,m},m}}{H}.
\]
By the multiplicative concentration argument we derived in~\cref{equ:conditional_bound}, with probability at least $1-\delta/64$,
\begin{align}
V_0^{*,m}-V_0^{\overline\pi^m,m}
&\leq\frac2{KL}\sum_{\ell=1}^L\sum_{k=1}^K
\bracket{V_1^{*,m}(s_1^{\ell,k})
-V_1^{\pi^{\ell,k,m},m}(s_1^{\ell,k})}
+\frac{2H}{KL}\log\frac{64M}{\delta}\nonumber\\
&\leq\frac2L\sum_{\ell=1}^L(\hat v_{\ell,m}^+-v_{\ell,m})
+\frac{2H}{KL}\log\frac{64M}{\delta}.
\label{eq:initial-state-certificate-transfer}
\end{align}
Here the second inequality is on the optimism event.

By Lemma~\ref{lemma:pessimism} and paying an additional failure probability
$\delta/64$, we have that 
\begin{align*}
V_0^{*,m}-V_0^{\overline\pi^m,m}
&\leq\frac2L\sum_{\ell=1}^L(\hat v_{\ell,m}^+-\hat v_{\ell,m}^-)
+\frac{2H}{L}\log\frac{64M}{\delta}
+\frac{2H}{KL}\log\frac{64M}{\delta}.
\end{align*}
Combining this with~\eqref{eq:average-cumulative-certificate} yields
\begin{align}
\max_{m\in[M]}\bracket{V_0^{*,m}-V_0^{\overline\pi^m,m}}
&\lesssim\sqrt{\frac{SAH^3}{KL}\log_2^3(2KL)\log\frac M{\delta'}}
+\frac{SAH^2}{K}\log\frac1{\delta''}\nonumber\\
&\quad+\frac HL\log\frac{64M}{\delta}
+\frac H{KL}\log\frac{64M}{\delta} = O(\epsilon).
\label{eq:combined-policy-error}
\end{align}
The proof of the optimality is finished.

\paragraph{Failure probability.}
The weighted optimistic bound and its underlying optimism event
together cost $(4SAHKLM+5)\delta'+\delta/64$.
The conditional mean evaluation bound costs $SAHL\delta''$.
The weight potential, cumulative evaluation overestimation, and
initial-state concentration each cost $\delta/64$.
Finally, Lemma~\ref{lemma:replay_count_bound} costs $2SAHL\delta''$.
On its success event, all retained quotas in the auxiliary run are
filled, and the coupling ensures that its entire history and output
agree with Algorithm~\ref{alg:main}. We intersect with this event only
after deriving the preceding conditional moment bounds.
Counting optimism only once, the total failure probability is at most
\begin{equation}
(4SAHKLM+5)\delta'+3SAHL\delta''+\frac{\delta}{16}.
\label{eq:total-failure-probability}
\end{equation}
For the algorithm's choices
$\delta'=\delta/(64SAHKLM)$ and $\delta''=\delta/(64SAHL)$,
the local confidence ranges hold. Since $SAHKLM\geq4$,
\eqref{eq:total-failure-probability} is at most $49\delta/256<\delta$.
The tail levels $\delta/64$ used above do not change either local
confidence parameter.

\paragraph{Total episode complexity.}
Each of the $L$ phases uses $K$ learning episodes and $8K$ replay
episodes. The total is therefore
\[
9KL
=O\bracket{\frac{SAH^3\log M}{\epsilon^2}
\log^4\bracket{\frac{SAH\log M}{\min\set{\epsilon,1}\delta}}}.
\]
The ceilings are absorbed by the bounds above, and $C$ is universal.
With probability at least $1-\delta$, all returned policies are
$\epsilon$-optimal for their respective rewards. This completes the proof.

\section{Supplementary Proofs in Section~\ref{sec:minimax}}\label{app:regret_analysis}

We use the auxiliary run from Appendix~\ref{app:minimax_additional_proof}.
Index all learning episodes consecutively by $k\in[KL]$: a superscript
$k$ abbreviates $(\ell,j)$ when $k=(\ell-1)K+j$ with $j\in[K]$.
Thus $K$ still denotes the number of learning episodes per phase,
and $\sum_{k,h}$ below ranges over $k\in[KL]$ and $h\in[H]$.
The empirical rows $\widehat P_{h,s,a}^k$ and completed-batch sizes
$N_h^k(s,a)$ belong to the persistent learner; replay samples do not
update them.

\subsection{Proof of Lemma~\ref{lemma:reward_switching}}
\label{app:sec:proof_of_reward_switching}
Let $e_s\in\R^S$ denote the $s$-th coordinate vector. Write
$V_h^k:=\overline V_h^{k,m^k}$, $Q_h^k:=\overline Q_h^{k,m^k}$,
$r_h^k:=r_h^{m^k}$, and $b_h^k:=b_h^{k,m^k}$.
Throughout this proof, $0<\delta'\leq\min\{e^{-1},(2KL)^{-1}\}$.
Since the policy is greedy with respect to $Q_h^k$,
\begin{align*}
V_h^k(s_h^k)
&=Q_h^k(s_h^k,a_h^k)\leq r_h^k(s_h^k,a_h^k)
+\langle\widehat P_{h,s_h^k,a_h^k}^k,V_{h+1}^k\rangle
+b_h^k(s_h^k,a_h^k).
\end{align*}
Adding and subtracting the true transition and telescoping gives
\begin{align}
&\sum_{k=1}^{KL}\bigl(V_1^k(s_1^k)-V_1^{\pi^k,m^k}(s_1^k)\bigr)\nonumber\\
&\leq \sum_{k,h}b_h^k(s_h^k,a_h^k)
+\sum_{k,h}\langle\widehat P_{h,s_h^k,a_h^k}^k-P_{h,s_h^k,a_h^k},V_{h+1}^k\rangle\nonumber\\
&\quad+\sum_{k,h}\langle P_{h,s_h^k,a_h^k}-e_{s_{h+1}^k},V_{h+1}^k\rangle
+\sum_{k=1}^{KL}\left(\sum_{h=1}^H r_h^k(s_h^k,a_h^k)-V_1^{\pi^k,m^k}(s_1^k)\right).
\label{eq:reward-switching-decomposition}
\end{align}
The last two terms are martingale sums. Hoeffding--Azuma bounds them by
$2\sqrt{KLH^3\log(1/\delta')}$ and
$2H\sqrt{KL\log(1/\delta')}$, respectively, each with failure probability
at most $\delta'$. For the latter sum, condition on the history, $m^k$,
and $s_1^k$ before the episode trajectory is drawn.
The forward histories also contain all completed auxiliary replay
outcomes and weight updates. The martingale sums remain unchanged
during replay, and the next learning transition still has conditional
law $P_{h,s_h^k,a_h^k}$.

If $KL$ satisfies~\eqref{eq:learning-large-K},
Lemmas~\ref{lemma:bd_bonus} and~\ref{lemma:diva} bound the first two
terms on the same event, with failure probability at most $3\delta'$.
Consequently, with probability at least $1-5\delta'$,
\[
\sum_{k=1}^{KL}\bigl(V_1^k(s_1^k)-V_1^{\pi^k,m^k}(s_1^k)\bigr)
\lesssim
\sqrt{SAH^3KL\log_2^3(2KL)\log\frac{M}{\delta'}}.
\]
If $KL$ is below the sufficiently large universal threshold
in~\eqref{eq:learning-large-K}, the deterministic bound on this
surplus is $HKL$, which is itself at most a universal constant times
the displayed square-root bound. Thus, for every $KL\geq1$,
\[
\sum_{k=1}^{KL}\bigl(V_1^k(s_1^k)-V_1^{\pi^k,m^k}(s_1^k)\bigr)
\lesssim
\min\left\{
\sqrt{SAH^3KL\log_2^3(2KL)\log\frac{M}{\delta'}},\ HKL
\right\}
\]
with probability at least $1-5\delta'$. The bound $HKL$ is used only
for the value surplus, not for the sum of untruncated bonuses.

\subsection{Profile-Based Concentration Lemma for Multi Reward MDPs}
Following~\cite{zhang2024settling}, couple the entire learning process
from initialization to an independent stream of next-state samples
for each $(h,s,a)$: the $i$-th learning visit to this row uses the
$i$-th sample in its stream. These streams are independent across rows
and stages and independent of all replay streams. We establish the
profile event globally, without conditioning on an intermediate phase's
history; batches and unfinished batch histograms persist across phases.
Let $\widehat P_{h,s,a}^{(j)}$ be the empirical transition of the
$j$-th potential batch, defined even if it is never completed.
Batch $1$ contains the first sample; for
$j\geq2$, batch $j$ contains samples $2^{j-2}+1,\ldots,2^{j-1}$.
Thus the first two batches both have size $1$ but are distinct.
Let $\widehat P_{h,s,a}^{(0)}$ be the fixed initial transition row.

Define the realized profile $\widehat{\mathcal I}_{h,s,a}^k$ to be
$0$ if no batch has been completed before episode $k$, and otherwise
the index of the most recently completed batch. In particular,
\[
\widehat P_{h,s,a}^k
=\widehat P_{h,s,a}^{(\widehat{\mathcal I}_{h,s,a}^k)}.
\]
The realized profile belongs to
\[
\mathcal C:=\left\{
\mathcal I=(\mathcal I^1,\ldots,\mathcal I^{KL}):
\mathcal I^1\leq\cdots\leq\mathcal I^{KL},\quad
\mathcal I^k\in\{0,\ldots,1+\lfloor\log_2(KL)\rfloor\}^{SAH}
\right\}.
\]
Each coordinate can be specified by the episode at which it first
reaches each positive batch index, using $KL+1$ if it never does.
Consequently,
\[
|\mathcal C|\leq
(KL+1)^{SAH(1+\lfloor\log_2(KL)\rfloor)}.
\]

\begin{lemma}\label{lemma:context_con}
Suppose $KL\geq\max\{2,SAH\}$. For every deterministic
$\mathcal I\in\mathcal C$, let $\mathcal X_{h,\mathcal I}\subseteq[0,H]^S$
contain $0$ and have at most $Z\geq2$ labeled vectors. Each labeled
vector must be a deterministic function of $\mathcal I$, the known
rewards, and the transition streams at stages strictly greater than
$h$. With probability at least $1-\delta'$, the bound
\begin{align}
&\sum_{k,h}\langle\widehat P_{h,s_h^k,a_h^k}^k-P_{h,s_h^k,a_h^k},X_h^k\rangle
\nonumber\\
&\lesssim
\sqrt{\log_2(2KL)\sum_{k,h}\mathbb V(P_{h,s_h^k,a_h^k},X_h^k)
\left(SAH\log_2^2(2KL)\log(2Z)+\log\frac1{\delta'}\right)}
\nonumber\\
&\quad+H\log_2(2KL)
\left(SAH\log_2^2(2KL)\log(2Z)+\log\frac1{\delta'}\right)
\label{eq:profile-concentration}
\end{align}
holds uniformly for all choices
$X_h^k\in\mathcal X_{h,\widehat{\mathcal I}}$.
\end{lemma}
\begin{proof}
First fix a deterministic profile $\mathcal I$, a batch index $j\geq1$,
and one deterministic candidate label for each $(h,s,a)$. Write
$X_{h,s,a}$ for the resulting vector. Reveal the streams backward in
$h$, starting with $h=H$. Before revealing the batch-$j$ samples at
stage $h$, every $X_{h,s,a}$ is measurable because all streams at
stages greater than $h$ have already been revealed. Each centered
sample contributes
\[
\frac{X_{h,s,a}(s')-\langle P_{h,s,a},X_{h,s,a}\rangle}
{\max\{1,2^{j-2}\}},
\]
which has conditional mean zero, absolute value at most
$H/\max\{1,2^{j-2}\}$, and conditional variance
$\mathbb V(P_{h,s,a},X_{h,s,a})/\max\{1,2^{j-2}\}^2$.
After each stage, reveal its remaining samples with zero increments
before moving to the preceding stage. This constructs a martingale
with total predictable variance
$\sum_{h,s,a}\mathbb V(P_{h,s,a},X_{h,s,a})/\max\{1,2^{j-2}\}$.
Freedman's inequality (Lemma~\ref{lemma:self-norm}), with variance
peeling, therefore gives the corresponding Bernstein bound for
$\sum_{h,s,a}\langle\widehat P_{h,s,a}^{(j)}-P_{h,s,a},X_{h,s,a}\rangle$.

There are at most $Z^{SAH}$ deterministic label assignments for each
profile and at most $1+\lfloor\log_2(KL)\rfloor$ batches. Allocate failure
probability $\delta'/[(1+\lfloor\log_2(KL)\rfloor)|\mathcal C|Z^{SAH}]$
to each such choice. The logarithmic cost of this union bound and the
variance peeling is at most a universal constant times
\[
SAH\log_2^2(2KL)\log(2Z)+\log\frac1{\delta'};
\]
here $KL\geq SAH$ also bounds the peeling factor. Thus, simultaneously
for every profile, batch, and candidate-label assignment,
\begin{align}
&\sum_{h,s,a}\langle\widehat P_{h,s,a}^{(j)}-P_{h,s,a},X_{h,s,a}\rangle
\nonumber\\
&\lesssim
\sqrt{\frac{\sum_{h,s,a}\mathbb V(P_{h,s,a},X_{h,s,a})}
{\max\{1,2^{j-2}\}}
\left(SAH\log_2^2(2KL)\log(2Z)+\log\frac1{\delta'}\right)}
\nonumber\\
&\quad+\frac{H}{\max\{1,2^{j-2}\}}
\left(SAH\log_2^2(2KL)\log(2Z)+\log\frac1{\delta'}\right).
\label{eq:xx1-aux-123}
\end{align}
Only after establishing this uniform event do we substitute the
realized profile $\widehat{\mathcal I}$.

A batch-$j$ model is used on at most $2^{j-1}$ subsequent visits to
its row. For each $i\leq2^{j-1}$, let
$\widetilde X_{h,s,a}^{j,i}$ be the vector used on the $i$-th such
visit, padding with $0$ if it does not occur. Although these choices
are random, they are covered by the uniform event above. Visits
with no completed batch contribute at most $SAH^2$ in total.
The remaining error decomposes as
\[
\sum_{j=1}^{1+\lfloor\log_2(KL)\rfloor}
\sum_{i=1}^{2^{j-1}}\sum_{h,s,a}
\langle\widehat P_{h,s,a}^{(j)}-P_{h,s,a},
\widetilde X_{h,s,a}^{j,i}\rangle.
\]
Apply~\eqref{eq:xx1-aux-123} to each summand and use Cauchy--Schwarz,
together with
\[
\sum_{j=1}^{1+\lfloor\log_2(KL)\rfloor}
\frac{2^{j-1}}{\max\{1,2^{j-2}\}}
\leq 2(1+\lfloor\log_2(KL)\rfloor)\leq2\log_2(2KL).
\]
Every observed visit appears once in the resulting variance sum.
The initial contribution $SAH^2$ is absorbed by the linear term in
\eqref{eq:profile-concentration}, proving the claim.
\end{proof}

\paragraph{Auxiliary optimistic values for fixed profiles.}
For a deterministic $\mathcal I\in\mathcal C$, run the backward
updates of Algorithm~\ref{alg:mvp} with
$\widehat P_{h,s,a}^k$ replaced by
$\widehat P_{h,s,a}^{(\mathcal I_{h,s,a}^k)}$ and with the corresponding
batch size ($0$ for index $0$, $1$ for index $1$, and $2^{j-2}$ for
index $j\geq2$). In particular, use the same bonus formula
\eqref{eq:update1}, including its $\max\{N,1\}$ denominators, and set
\begin{align}
\overline V_{H+1}^{k,m,\mathcal I}(s)&=0,\label{eq:term}\\
\overline V_h^{k,m,\mathcal I}(s)
&=\max_{a\in\cA}\min\left\{
 r_h^m(s,a)
 +\langle\widehat P_{h,s,a}^{(\mathcal I_{h,s,a}^k)},
 \overline V_{h+1}^{k,m,\mathcal I}\rangle
 +b_h^{k,m,\mathcal I}(s,a),\ H\right\}.
\label{eq:backup}
\end{align}
The family
\[
\mathcal X_{h,\mathcal I}
=\{\overline V_{h+1}^{k,m,\mathcal I}:k\in[KL],\ m\in[M]\}
\cup\{0\}
\]
has cardinality at most $Z=KLM+1$ and satisfies the
measurability assumption of Lemma~\ref{lemma:context_con}.
For the realized profile, backward induction gives
$\overline V_h^{k,m,\widehat{\mathcal I}}=\overline V_h^{k,m}$.
For fixed $\mathcal I$, these reward-specific backups do not depend on
$p_\ell$ or on replay outcomes. The realized profile and candidate labels
may depend on past auxiliary evaluations, but are covered by the same
uniform event. Hence adaptive reward weights require no additional
union bound.
The same properties hold for the family of vectors
$(\overline V_{h+1}^{k,m,\mathcal I})^2/H$, together with $0$.

\subsection{Statement and Proof of Lemma~\ref{lemma:bd_bonus}}
\begin{lemma}\label{lemma:bd_bonus}
Suppose $0<\delta'\leq\min\{e^{-1},(2KL)^{-1}\}$ and
\begin{equation}
KL\gtrsim SAH\log_2^3(2KL)\log\frac{M}{\delta'},
\label{eq:learning-large-K}
\end{equation}
with a sufficiently large universal implied constant.
With probability at least $1-3\delta'$,
\[
\sum_{k,h}b_h^k(s_h^k,a_h^k)
\lesssim\sqrt{SAH^3KL\log_2(2KL)\log\frac1{\delta'}}.
\]
On the same event,
$\sum_{k,h}\mathbb V(P_{h,s_h^k,a_h^k},V_{h+1}^k)\lesssim KLH^2$
and
$\sum_{k,h}\mathbb V(\widehat P_{h,s_h^k,a_h^k}^k,V_{h+1}^k)\lesssim KLH^2$.
\end{lemma}
\begin{proof}
Retain the abbreviations
\[
\widehat{\mathtt{Var}}:=\sum_{k,h}\mathbb V(\widehat P_{h,s_h^k,a_h^k}^k,V_{h+1}^k),
\qquad
\mathtt{Var}:=\sum_{k,h}\mathbb V(P_{h,s_h^k,a_h^k},V_{h+1}^k),
\]
and $\mathtt{BonusSum}:=\sum_{k,h}b_h^k(s_h^k,a_h^k)$.
By the bonus formula, Cauchy--Schwarz, and
Lemma~\ref{lemma:doubling} applied to all $KL$ learning episodes,
\begin{align}
\mathtt{BonusSum}
&=\frac{460}{9}\sum_{k,h}
\sqrt{\frac{\mathbb V(\widehat P_{h,s_h^k,a_h^k}^k,V_{h+1}^k)
\log(1/\delta')}{\max\{N_h^k(s_h^k,a_h^k),1\}}}
+\frac{544H}{9}\sum_{k,h}\frac{\log(1/\delta')}
{\max\{N_h^k(s_h^k,a_h^k),1\}}\nonumber\\
&\leq\frac{460}{9}\sqrt{2SAH\log_2(2KL)\log(1/\delta')\,
\widehat{\mathtt{Var}}}
+\frac{1088}{9}SAH^2\log_2(2KL)\log\frac1{\delta'}.
\label{eq:boundt2o-temp}
\end{align}
The variance identities, telescoping over $h$, and
$V_h^k(s_h^k)\leq r_h^k(s_h^k,a_h^k)
+\langle\widehat P_{h,s_h^k,a_h^k}^k,V_{h+1}^k\rangle
+b_h^k(s_h^k,a_h^k)$ imply
\begin{align*}
\widehat{\mathtt{Var}}
&\leq\sum_{k,h}\langle\widehat P_{h,s_h^k,a_h^k}^k-P_{h,s_h^k,a_h^k},(V_{h+1}^k)^2\rangle\\
&\quad+\sum_{k,h}\langle P_{h,s_h^k,a_h^k}-e_{s_{h+1}^k},(V_{h+1}^k)^2\rangle
+2H\mathtt{BonusSum}+2KLH^2,\\
\mathtt{Var}
&\leq\sum_{k,h}\langle P_{h,s_h^k,a_h^k}-e_{s_{h+1}^k},(V_{h+1}^k)^2\rangle\\
&\quad+2H\sum_{k,h}\max\{\langle\widehat P_{h,s_h^k,a_h^k}^k-P_{h,s_h^k,a_h^k},V_{h+1}^k\rangle,0\}
+2H\mathtt{BonusSum}+2KLH^2.
\end{align*}
Here we used $0\leq V_h^k\leq H$, $\sum_{k,h}r_h^k(s_h^k,a_h^k)\leq KLH$,
and $x^2-y^2\leq2H\max\{x-y,0\}$ for $x,y\in[0,H]$.

Apply Lemma~\ref{lemma:context_con} to the auxiliary value family and
to its squared, normalized family. The two uniform events have total
failure probability at most $2\delta'$. Inclusion of $0$ allows the
first event to bound the positive parts of the model errors, by
selecting $0$ whenever the inner product is negative; it also bounds
their signed sum. For the second family use
$\mathbb V(p,V^2/H)\leq4\mathbb V(p,V)$.
Because $\delta'\leq(2KL)^{-1}$,
$\log(2(KLM+1))\lesssim\log(M/\delta')$.
The remaining squared-value martingale is at most
$2H^2\sqrt{KLH\log(1/\delta')}\leq KLH^2$, except on an event of
probability $\delta'$; the last inequality follows from
\eqref{eq:learning-large-K} with a sufficiently large implied constant.
Thus, on a common event of probability at least $1-3\delta'$,
\begin{align*}
\max\{\mathtt{Var},\widehat{\mathtt{Var}}\}
\lesssim{}&
H\sqrt{SAH\log_2^3(2KL)\log\frac{M}{\delta'}\,
\mathtt{Var}}
+SAH^3\log_2^3(2KL)\log\frac{M}{\delta'}\\
&+KLH^2+H\mathtt{BonusSum}.
\end{align*}
Insert~\eqref{eq:boundt2o-temp} and apply Young's inequality to absorb
the terms involving $\sqrt{\mathtt{Var}}$ and
$\sqrt{\widehat{\mathtt{Var}}}$ into the left-hand side. This gives
\[
\max\{\mathtt{Var},\widehat{\mathtt{Var}}\}
\lesssim KLH^2+SAH^3\log_2^3(2KL)\log\frac{M}{\delta'}
\lesssim KLH^2,
\]
where the last step uses~\eqref{eq:learning-large-K}.
Finally, substitute into~\eqref{eq:boundt2o-temp}; the same condition
absorbs its linear term into the square-root term. This proves the claim.
\end{proof}

\subsection{Statement and Proof of Lemma~\ref{lemma:diva}}
\begin{lemma}\label{lemma:diva}
Under the hypotheses of Lemma~\ref{lemma:bd_bonus}, on the same event
of probability at least $1-3\delta'$,
\[
\sum_{k,h}\langle\widehat P_{h,s_h^k,a_h^k}^k-P_{h,s_h^k,a_h^k},V_{h+1}^k\rangle
\lesssim\sqrt{SAH^3KL\log_2^3(2KL)\log\frac{M}{\delta'}}.
\]
\end{lemma}
\begin{proof}
The value-family concentration event used in the preceding proof
already gives
\begin{align*}
&\sum_{k,h}\langle\widehat P_{h,s_h^k,a_h^k}^k-P_{h,s_h^k,a_h^k},V_{h+1}^k\rangle\\
&\lesssim
\sqrt{SAH\log_2^3(2KL)\log\frac{M}{\delta'}\,
\mathtt{Var}}
+SAH^2\log_2^3(2KL)\log\frac{M}{\delta'}.
\end{align*}
Use $\mathtt{Var}\lesssim KLH^2$ from Lemma~\ref{lemma:bd_bonus}
and absorb the linear term by~\eqref{eq:learning-large-K}.
No additional failure probability is incurred.
\end{proof}

\section{Numerical Experiments}
\label{app:sec:numerical_experiment}

This appendix provides the experimental setting and evaluation details for the two experiments in Section~\ref{sec:numerical_experiment}.

We consider random nonstationary finite-horizon MDPs with $S=10$,
$A=5$, $H=6$, and initial state $s_1=0$. For each stage $h\in[H]$ and state--action pair $(s,a)$, we independently sample a transition row as
\[
\alpha_{h,s,a}\sim\operatorname{Uniform}(0.15,1.2),
\qquad
P_h(\cdot\mid s,a)\sim
\operatorname{Dirichlet}(\alpha_{h,s,a}\mathbf{1}_{S}).
\]
Thus, transitions vary with $h$, $s$, and $a$, have full support, and
contain no absorbing states.
Independently, each reward function is generated by sampling
\[
u_{m,h,s}\sim\operatorname{Beta}(0.25,0.25),
\qquad r_h^m(s,a)=u_{m,h,s}\quad\text{for all }a.
\]
Rewards vary across reward functions, stages, and states, but are shared across
actions at each state and stage. The transition and reward tables remain
fixed within each instance. Returns are undiscounted sums over $H$ steps,
with no additional terminal reward.

\subsection{Experiment 1: Policy error versus episode budget}
\label{sec:app:episode_error_evaluation}

We fix $M=16$ and test total episode budgets
$B\in\{2^9,2^{10},\ldots,2^{17}\}$.
For each budget, we run each implementation from scratch using 12
independent environment--learner seed pairs, shared across algorithms
and budgets. Optimal and learned-policy values are computed exactly
by backward dynamic programming with the true kernel. We take the maximum policy error over the $M$ rewards within each run and then average over the 12 runs, reporting $\overline{\mathrm{Error}}_{12}(B,16)$ as defined in the main text.

For every algorithm, the number of episodes in Figure~\ref{fig:episode_error} counts all environment episodes actually used across all stages. We denote this total by $B$, consistently with the main text; no additional interaction cost is excluded.
For our algorithm, this includes the initial random warmup and all subsequent learning episodes.
Pointwise 95\% confidence intervals are obtained from the 2.5th and 97.5th percentiles of 10,000 bootstrap means, resampling seed pairs with replacement and preserving their pairing across algorithms and budgets. The comparisons and pointwise intervals concern the mean worst-reward error under these selected configurations; the bootstrap holds configurations fixed and does not account for their selection.

\subsection{Experiment 2: Required episode budget versus reward count}
\label{sec:app:reward_count_evaluation}

We vary $M\in\{4,6,8,12,16,24,32,48,64,96,128\}$ and use eight independent environment--learner seed pairs, with algorithm settings fixed as in Experiment 1. For each environment, we keep the transition kernel fixed and generate 128 reward functions; the instance with $M$ rewards uses the first $M$ functions. Each algorithm is rerun for every $(B,M)$ setting, using the same exact policy evaluation and episode accounting as above.

For each $M$, we estimate the required number of episodes $B$ to achieve $\overline{\mathrm{Error}}_{8}(B,M)\le0.1$, using the same total count of episodes actually used by each algorithm. We first identify the smallest tested budget at which the mean error, and the mean errors at all larger tested budgets, are at most $0.1$. We then estimate the crossing by linear interpolation of log error against log budget between that point and the preceding tested budget. Budget grids are refined near the crossing for each algorithm. Pointwise 95\% confidence intervals use 10,000 paired bootstrap resamples of the eight seed pairs, recomputing the mean error curve and interpolated crossing in each resample and taking their 2.5th and 97.5th percentiles. The plotted quantity is therefore the estimated crossing of the mean error curve, rather than the average of individual runs' crossing budgets.
These pointwise intervals quantify empirical variability under the selected configurations, rather than the simultaneous PAC guarantee.





%
%
%

\section{Lower Bound Construction}
\label{app:lb}

We establish a lower bound for learning policies for a finite collection of reward
functions fixed and revealed before exploration. Our proof
combines a finite-packing and decoding argument
\citep[Appendix~D.4]{jin2020reward} with a time-dependent tree embedding
\citep[Appendix~C.3]{ridel2026improvedboundsrewardagnosticrewardfree}. The construction places independent
hard transitions at $\Theta(SH)$ time--leaf pairs while sharing the same
physical states (see Figure~\ref{fig:lb-construction}). 

\begin{proposition}
\label{prop:finite-known-reward-lb}
There exist universal constants $c_s,c_{\mathrm{lb}}>0$ such that the
following holds. Suppose that $A\geq2$, $L\geq16$, $
 \log(AL)+1\leq c_sS$, $ H\geq4\bigl(1+\lceil\log_2S\rceil\bigr)$ and $
 0<\epsilon\leq\frac1{192}$. 
There exist a class $\mathcal P_L$ of time-inhomogeneous transition
kernels on at most $S$ states, with $A$ actions and horizon $H$, and a
fixed reward class
$\mathcal R_L\subset[0,1]^{\mathcal S\times\mathcal A\times[H]}$
satisfying $|\mathcal R_L|\leq SHA^2L^2$, with the following property.
Any algorithm that is given $\mathcal R_L$ before exploration, stops
after a possibly random number $K$ of online episodes, and satisfies
\[
 \mathbb P_P\!\left(
  V_{P,r}^{\star}-V_{P,r}^{\widehat\pi_r}\leq\epsilon
  \quad\text{for every }r\in\mathcal R_L
 \right)\geq\frac12
 \quad\text{for every }P\in\mathcal P_L
\]
must obey
\[
 \sup_{P\in\mathcal P_L}\mathbb E_P[K]
 \geq c_{\mathrm{lb}}\frac{SAH^3\log L}{\epsilon^2}.
\]
\end{proposition}

\begin{figure}[ht]
    \centering
    \includegraphics[width=0.98\linewidth]{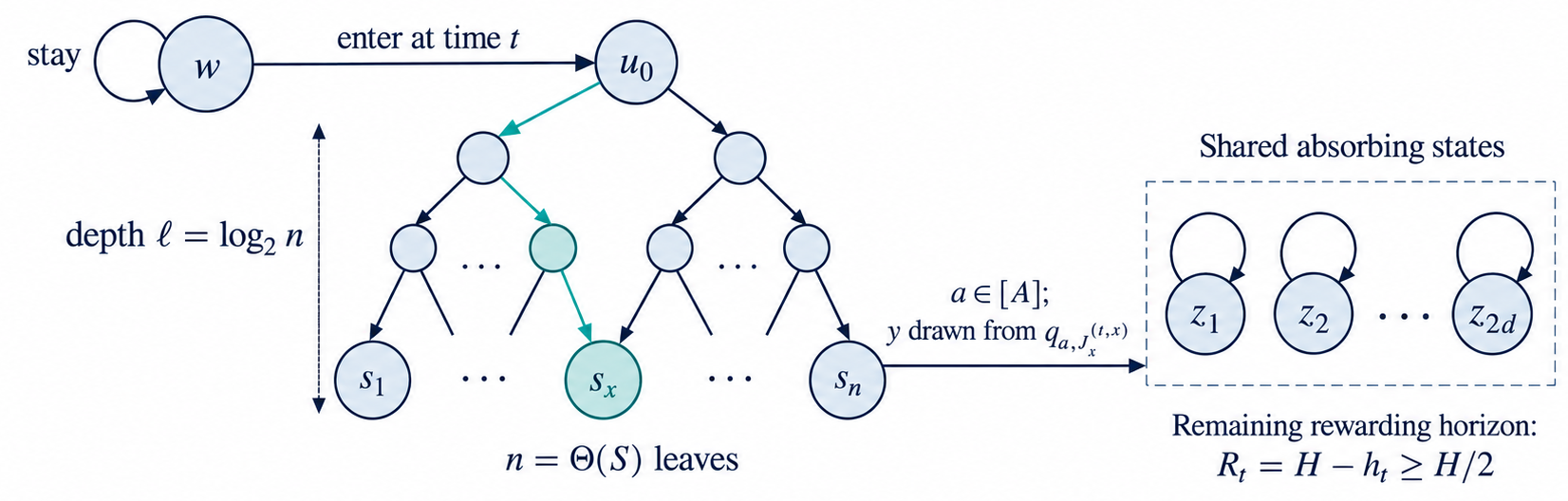}

    \medskip
    \begin{minipage}[t]{0.485\linewidth}
        \centering
        \includegraphics[width=\linewidth]{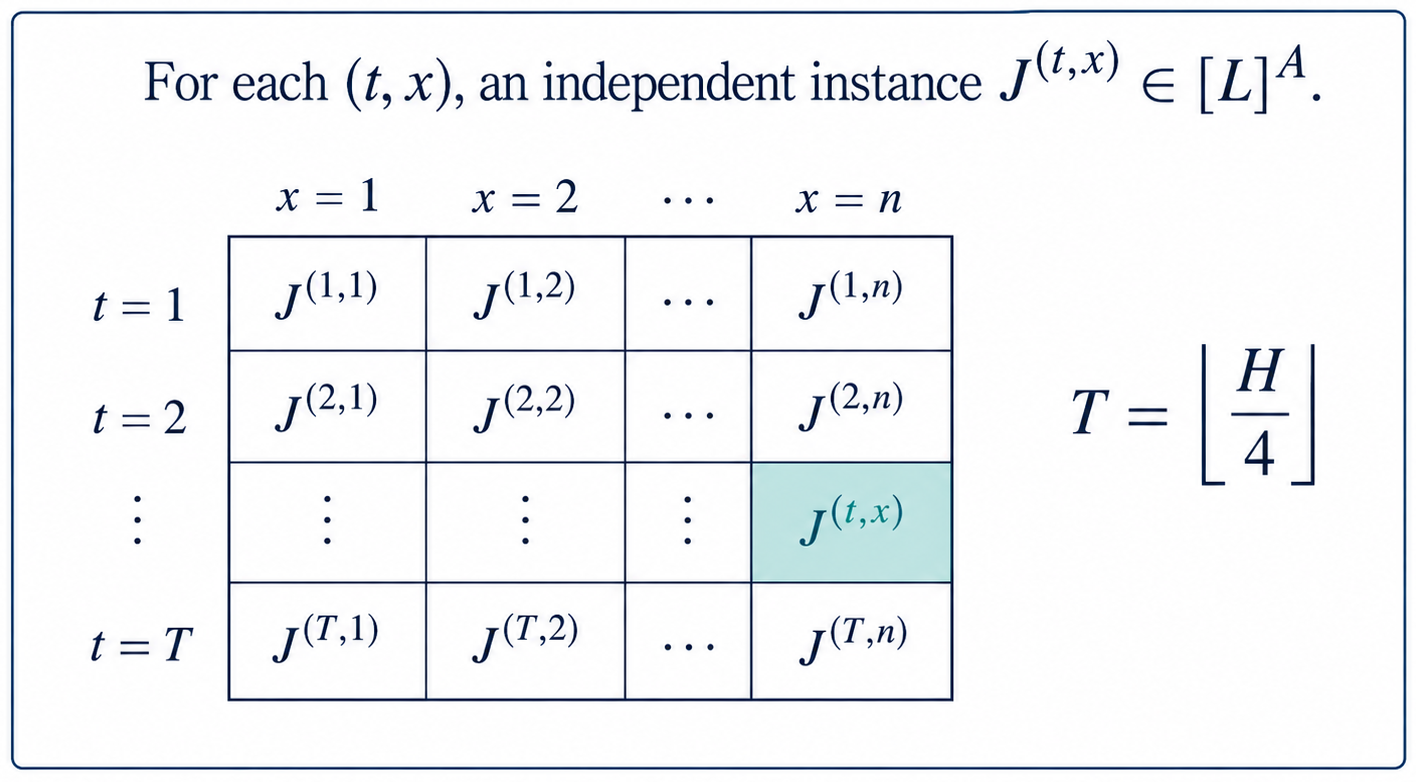}
    \end{minipage}\hfill
    \begin{minipage}[t]{0.485\linewidth}
        \centering
        \includegraphics[width=\linewidth]{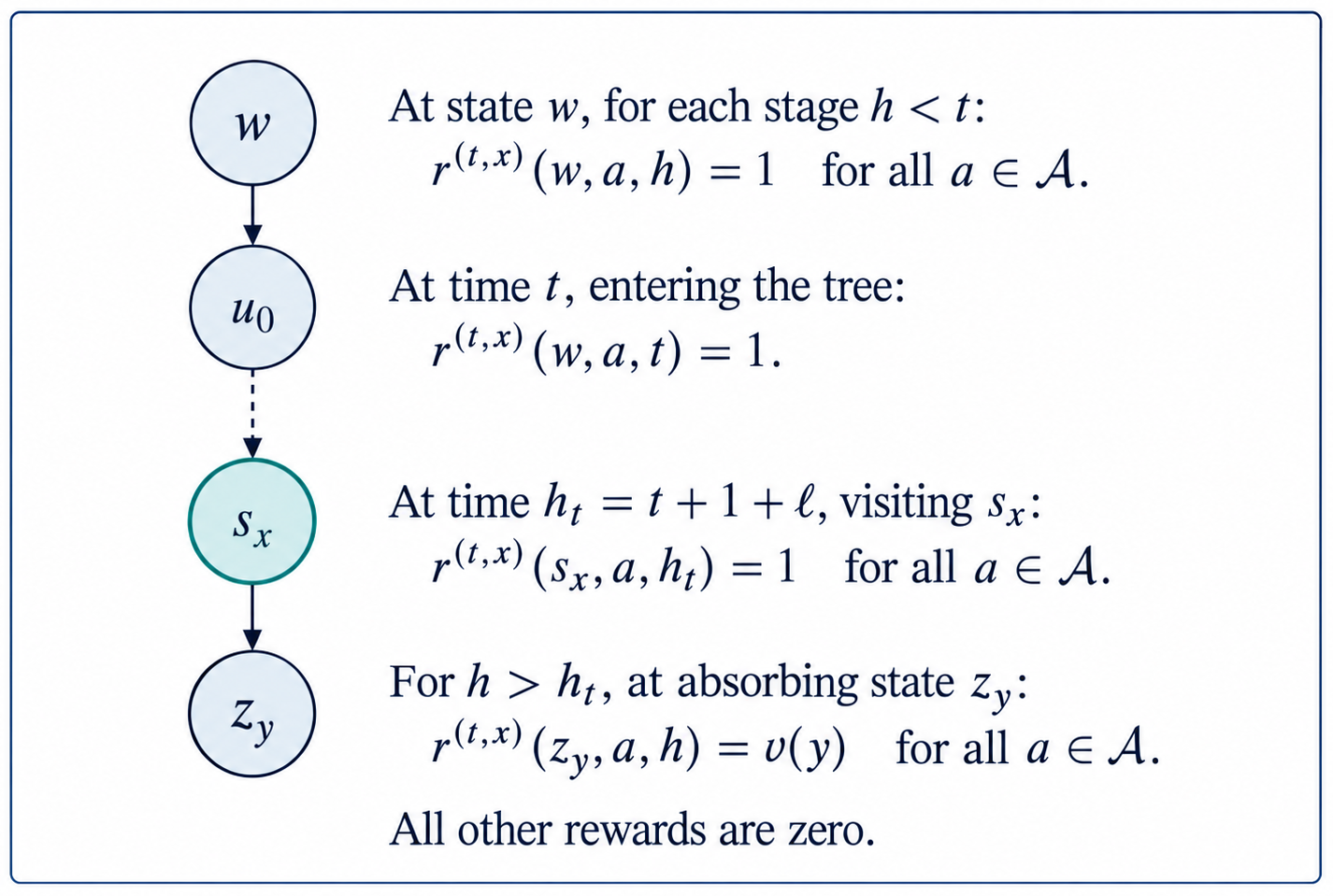}
    \end{minipage}
    \caption{The lower-bound construction.
    \textbf{(a)} A waiting state and a deterministic binary tree lead to
    shared absorbing states. At stage $h_t=t+1+\ell$, action $a$ at leaf
    $s_x$ has transition distribution $q_{a,J_a^{(t,x)}}$.
    \textbf{(b)} Each time--leaf pair $(t,x)$ has its own hidden tuple
    $J^{(t,x)}\in[L]^A$; different times share the same physical states.
    \textbf{(c)} The publicly known reward $r_{t,x,\nu}$ localizes the
    policy at the designated pair before evaluating the absorbing-state
    reward $\nu$. The entry reward at stage $t$ applies only to actions
    $a\neq1$, whereas action $1$ stays at $w$.}
    \label{fig:lb-construction}
\end{figure}

\begin{proof}
Without loss of generality, we assume $\sup_{P\in\mathcal P_L}\mathbb E_P[K]$ is finite. Set $
 \alpha=\frac{96\epsilon}{H}$, so $ H\alpha\leq\frac12.$

\paragraph{Step 1: the construction of leaf nodes.}
Fix $\gamma=1/10$ and
$d=\lceil1000(\log(AL)+1)\rceil$.
The balanced packing lemma
\citep[Lemma~D.6]{jin2020reward} provides vectors
$v_{a,j}\in\{-1,1\}^{2d}$, indexed by $(a,j)\in[A]\times[L]$,
such that
\begin{equation}
 \langle v_{a,j},\mathbf1\rangle=0,\qquad
 |\langle v_{a,j},v_{a',j'}\rangle|<2\gamma d
 \quad\text{whenever }(a,j)\neq(a',j').
 \label{eq:lb-packing}
\end{equation}
Our choice of $d$ satisfies that lemma's condition $
 2\log L\leq\gamma^2d-\log(4d)-2\log A$. 
Taking $c_s$ sufficiently small also ensures $d\leq S/4$.
Define
\begin{equation}
 q_0=\frac{\mathbf1}{2d},\qquad
 q_{a,j}=\frac{\mathbf1+\alpha v_{a,j}}{2d},
 \label{eq:lb-q}
\end{equation}
and the finite collection of decoder rewards
\begin{equation}
 \mathcal V_L=
 \left\{
   \nu_{a,a',j,j'}
   :=\frac12\mathbf1+\frac13v_{a,j}+\frac16v_{a',j'}:
   a\neq a',\ j,j'\in[L]
 \right\}.
 \label{eq:lb-decoder-rewards}
\end{equation}
Then $q_{a,j}\in\Delta(2d)$,
$\mathcal V_L\subset[0,1]^{2d}$, and
$|\mathcal V_L|\leq A(A-1)L^2$.
For a hidden tuple $J\in[L]^A$, action $a$ produces one observation
with distribution $q_{a,J_a}$.

The decoding argument
\citep[proof of Lemma~D.8 and Lemma~D.11]{jin2020reward}
has the following property: if distributions
$\{\lambda_\nu\in\Delta(A):\nu\in\mathcal V_L\}$ satisfy
\begin{equation}
 \max_a q_{a,J_a}^{\top}\nu
 -\sum_a\lambda_\nu(a)q_{a,J_a}^{\top}\nu
 \leq\frac\alpha{24}
 \quad\text{for every }\nu\in\mathcal V_L,
 \label{eq:lb-decoder-condition}
\end{equation}
then they determine $J$ exactly. More precisely, $J_a$ is the unique
$j\in[L]$ such that
\begin{equation}
 \lambda_{\nu_{a,a',j,j'}}(a)>\frac12
 \quad\text{for every }a'\neq a\text{ and }j'\in[L].
 \label{eq:lb-decoder}
\end{equation}

\paragraph{Step 2: embedding at time--leaf pairs.}
Let $n$ be the largest power of two not exceeding $S/4$, and write
$\ell=\log_2n$ and $T=\lfloor H/4\rfloor$.
Construct an initial waiting state $w$, a complete binary tree with
root $u_0$ and leaves $s_1,\ldots,s_n$, and shared absorbing states
$z_1,\ldots,z_{2d}$.

At $w$, action $1$ stays and every other action enters $u_0$.
At each internal tree state, action $1$ moves to the left child and
every other action moves to the right child. These transitions are
deterministic and known. Thus, $A\geq2$ suffices, while all $A$
actions are available at every leaf.

For each $(t,x)\in[T]\times[n]$, assign a hidden tuple
$J^{(t,x)}\in[L]^A$. A trajectory entering the tree at stage $t$
reaches a leaf at stage $
 h_t=t+1+\ell$.
At this stage, define
\begin{equation}
 P_{h_t}(z_y\mid s_x,a)=q_{a,J_a^{(t,x)}}(y),
 \qquad a\in[A],\quad y\in[2d].
 \label{eq:lb-hard-transition}
\end{equation}
At every leaf stage outside $\{h_t:t\in[T]\}$, the transition
distribution is $q_0$. Each $z_y$ self-loops under every action.
This defines $\mathcal P_L$ as the hidden tuples vary.
Figure~\ref{fig:lb-construction}(a)--(b) illustrates the shared
physical MDP and its time--leaf indexing.

We use a separate waiting state in the tree embedding
\citep[Appendix~C.3]{ridel2026improvedboundsrewardagnosticrewardfree}, so that two actions suffice
for both waiting and routing. Its state count is $
 1+(2n-1)+2d=2n+2d\leq S$. 
Moreover, $n=\Theta(S)$, $T=\Theta(H)$, and the horizon condition gives $
 R_t:=H-h_t\geq H/2.$ 
Each episode contains at most one informative transition: after
reaching a leaf, the process is absorbed. The time indices introduce
no additional physical states.

\paragraph{Step 3: fixed rewards and localization.}
For each $(t,x,\nu)\in[T]\times[n]\times\mathcal V_L$, define the
reward illustrated in Figure~\ref{fig:lb-construction}(c):
\begin{equation}
 r_{t,x,\nu}(s,a,h)=
 \begin{cases}
  1,      & s=w,\ h<t,\\
  1,      & s=w,\ h=t,\ a\neq1,\\
  1,      & s=s_x,\ h=h_t,\\
  \nu(y), & s=z_y,\ h>h_t,\\
  0,      & \text{otherwise}.
 \end{cases}
 \label{eq:lb-known-rewards}
\end{equation}
Let $\mathcal R_L$ contain all these rewards. They depend only on
the fixed packing and known routing, not on the hidden tuples.
Consequently, the whole class may be revealed before exploration, and
\begin{equation}
 |\mathcal R_L|
 \leq TnA(A-1)L^2
 \leq SHA^2L^2.
 \label{eq:lb-reward-count}
\end{equation}

Fix $r=r_{t,x,\nu}$ and let $G=\{s_{h_t}=s_x\}$.
On $G$, the non-absorbing rewards sum to exactly $t+1$; on $G^c$,
they sum to at most $t$. Also, balance of the packing and
$\|q_{a,j}-q_0\|_1=\alpha$ imply
\begin{equation}
 q_0^{\top}\nu=\frac12,\qquad
 \left|q_{a,j}^{\top}\nu-\frac12\right|\leq\frac\alpha2.
 \label{eq:lb-center}
\end{equation}
The event $G$ is determined by the waiting and routing decisions,
before any random leaf transition. Thus, conditioning on $G^c$
does not select favorable absorbing-state outcomes. Since absorbing
rewards are zero through stage $h_t$, even an early exit gives at
most $R_t$ rewarded absorbing stages.

It follows that a policy reaching the target pair and choosing its
best action has value at least $
 t+1+\frac{R_t(1-\alpha)}2$, 
whereas the conditional value on $G^c$ is at most $
 t+\frac{R_t(1+\alpha)}2$. 
Their difference is at least
\begin{equation}
 1-R_t\alpha\geq1-H\alpha\geq\frac12.
 \label{eq:lb-localization-gap}
\end{equation}
In particular, an optimal policy reaches the designated pair with
probability one, and
\[
 V_{P,r}^{\star}
 =t+1+R_t\max_a q_{a,J_a^{(t,x)}}^{\top}\nu.
\]

For any policy $\pi$, put $p=\mathbb P_P^{\pi}(G)$ and let $\lambda$
be its conditional action distribution at $(s_x,h_t)$, arbitrary
when $p=0$. Define
\[
 \Delta_\nu(\lambda)
 =
 \max_a q_{a,J_a^{(t,x)}}^{\top}\nu
 -
 \sum_a\lambda(a)q_{a,J_a^{(t,x)}}^{\top}\nu.
\]
The conditional expected value on $G$ is
$V_{P,r}^{\star}-R_t\Delta_\nu(\lambda)$; on $G^c$, it is at most
$V_{P,r}^{\star}-1/2$. Therefore,
\begin{equation}
 V_{P,r}^{\star}-V_{P,r}^{\pi}
 \geq
 \frac{1-p}{2}
 +
 pR_t\Delta_\nu(\lambda).
 \label{eq:lb-localization-decomposition}
\end{equation}
If $\pi$ is $\epsilon$-optimal, then
\begin{equation}
 p\geq1-2\epsilon\geq\frac12,\qquad
 \Delta_\nu(\lambda)
 \leq\frac{\epsilon}{pR_t}
 \leq\frac{4\epsilon}{H}
 =\frac\alpha{24}.
 \label{eq:lb-local-accuracy}
\end{equation}

For Markov policies, $\lambda=\pi_{h_t}(\cdot\mid s_x)$; more
generally, it is determined by the output policy and the known
deterministic prefix. No unknown transition is needed to compute it.
Applying \eqref{eq:lb-decoder} separately at each pair shows that
simultaneous $\epsilon$-optimality for $\mathcal R_L$ permits exact
recovery of the entire hidden-index array.

\paragraph{Step 4: Putting it all together.}
Let $\Theta=[L]^{TnA}$. Each $\theta$ specifies the transition distributions for all $TnA$
state--action pairs associated with the $Tn$ time-expanded leaf nodes.
Let $\mathbb P_{\theta}$ denote this transition kernel. As a
comparison, let $\mathbb P_0$ be the reference transition kernel
when every hard transition is replaced by $q_0$, with the same
known reward class. Choose $\mathcal{P}_{L} = \{ \mathbb{P}_{\theta}| \theta \in \Theta\}\cup \{\mathbb{P}_0\}$.

Let $N_{t,x,a}$ count visits to the associated hard transition before the algorithm stops. Let $\mathbb{Q}_{\theta}$ and $\mathbb{Q}_{0}$ denote the probability laws of the algorithm's stopped interaction transcript under transition kernels $P_{\theta}$ and $P_{0}$, respectively, with the same publicly known reward class. The transcript includes the stopping time and the algorithm's output. The adaptive KL chain rule, applied to the stopped transcript
\citep[proof of Lemma~D.7]{jin2020reward}, gives
\begin{align}
 D_{\mathrm{KL}}\!\left(
   \mathbb Q_\theta
   \,\middle\|\,
   \mathbb Q_0
 \right)
 &\leq
 \sum_{t,x,a}
 \mathbb E_\theta[N_{t,x,a}]
 D_{\mathrm{KL}}\!\left(
   q_{a,\theta_{t,x,a}}\|q_0
 \right)
 \notag\\
 &\leq
 \alpha^2
 \mathbb E_\theta\!\left[
   \sum_{t,x,a}N_{t,x,a}
 \right]
 \leq
 \alpha^2\mathbb E_\theta[K].
 \label{eq:lb-transcript-kl}
\end{align}
Here $D_{\mathrm{KL}}(q_{a,j}\|q_0)\leq\alpha^2$ follows from
$\log(1+u)\leq u$ and \eqref{eq:lb-packing}.

Denote the uniform distribution on
$\Theta$ as $\mathrm{Uni}(\Theta)$ and choose $\theta$ following $\mathrm{Uni}(\Theta)$.
Recall that the algorithm provides an estimator $\widehat\theta$ with error probability at most $1/2$ under every $\theta$. 
By \citep[Lemma~D.10]{jin2020reward} and
\eqref{eq:lb-transcript-kl}, we have that 
\begin{align*}
 \alpha^2
 \mathbb E_{\theta\sim \mathrm{Uni}(\Theta)}\mathbb{E}_{\theta}
[K] \geq& \frac{1}{|\Theta|} \sum_{\theta \in \Theta} D_{\mathrm{KL}}( \mathbb{Q}_{\theta}|| \mathbb{Q}_0)
 \geq
 \frac12\log|\Theta|-\log2\\
 =& 
 \frac12TnA\log L-\log2\geq
 \frac14TnA\log L.
\end{align*}
Consequently,
\[
 \sup_{P\in\mathcal P_L}\mathbb E_P[K]
 \geq \mathbb{E}_{\theta \sim \mathrm{Uni}(\Theta)}\mathbb E_\theta[K]
 \geq
 c\frac{TnA\log L}{\alpha^2}
 \geq
 c'\frac{SAH^3\log L}{\epsilon^2},
\]
using $n=\Theta(S)$, $T=\Theta(H)$, and
$\alpha=96\epsilon/H$.
\end{proof}

\begin{corollary}[Dependence on the number of rewards]
\label{cor:finite-known-reward-lb-M}
For a sufficiently large universal constant $C$, suppose that $
 A\geq2$, $
 S\geq C(1+\log A)$, $
 H\geq4(1+\lceil\log_2S\rceil)$,
 $0<\epsilon\leq\frac1{192}$ 
and $M\geq(CSHA^2)^2$. There exists a fixed, publicly known class
of exactly $M$ reward functions for which any algorithm achieving
simultaneous $\epsilon$-optimality with probability at least $1/2$
requires
\[
 \Omega\!\left(
   \frac{SAH^3}{\epsilon^2}\min\{S,\log M\}
 \right)
\]
expected online episodes in the worst case. In particular, if
$\log M\leq S$, the lower bound becomes
\[
 \Omega\!\left(\frac{SAH^3\log M}{\epsilon^2}\right).
\]
\end{corollary}

\begin{proof}
Choose
\[
 L=
 \left\lfloor
 \min\left\{
   \sqrt{\frac{M}{SA^2H}},
   \exp\!\left(\frac{c_sS}{2}\right)
 \right\}
 \right\rfloor.
\]
For sufficiently large $C$, the assumptions imply $L\geq16$ and
$\log A+1\leq c_sS/2$, so that $\log(AL)+1\leq c_sS$.
Moreover, $|\mathcal R_L|\leq SHA^2L^2\leq M$, and
\[
 \log L
 =
 \Omega\!\left(
   \min\left\{S,\log\frac{M}{SA^2H}\right\}
 \right)
 =
 \Omega\!\left(\min\{S,\log M\}\right),
\]
where the last equality uses $M\geq (SA^2H)^2$.
The conclusion follows by Proposition~\ref{prop:finite-known-reward-lb}.

\end{proof}


\clearpage


\end{document}